\documentclass[11pt]{article}
\usepackage{amsmath,amsthm,amsfonts,amssymb,amscd}
\usepackage[top=1in,bottom=1in,left=1in,right=1in]{geometry}
\usepackage{soul}
\usepackage{graphicx,array}
\usepackage{subfigure,tabularx}
\usepackage[numbers]{natbib}
\usepackage{booktabs,caption}
\usepackage{etoolbox,verbatim,color}
\usepackage{url}
\usepackage{accents}
\usepackage{amsmath,amsthm,amsfonts,amssymb
,amscd}
\usepackage{multirow}

\usepackage{mathtools}

\usepackage{enumitem}

\usepackage[linesnumbered,ruled]{algorithm2e}
\usepackage{appendix}
\newtheorem{theorem}{Theorem}[section]

\newtheorem{assumption}[theorem]{Assumption}

\newtheorem{lemma}[theorem]{Lemma}

\theoremstyle{definition}
\newtheorem{definition}[theorem]{Definition}

\theoremstyle{definition}
\newtheorem{remark}[theorem]{Remark}

\theoremstyle{definition}

\newcommand{\E}{{\mathbb{E}}}

\newcommand{\mbfU}{\mathbf{U}}
\newcommand{\mbfW}{\mathbf{W}}

\newcommand{\mbfY}{\mathbf{Y}}

\usepackage[linesnumbered,ruled]{algorithm2e}

\newcommand{\argmin}{\ensuremath{\operatorname*{argmin}}}

\newcommand{\dist}{\ensuremath{\operatorname{dist}}}

\def\argmin{ \mathop{{\rm argmin}}}

\newcommand{\scal}[2]{\left\langle{#1}\,,\,{#2}\right\rangle}

\providecommand{\keywords}[1]
{
  \small	
  \textbf{\textit{Keywords---}} #1
}

\title{A Linearized Alternating Direction Multiplier Method for Federated Matrix Completion Problems}
\author{
    Patrick Hytla$^{1}$ \\
    \and
    Tran T. A. Nghia$^{2}$ \\
    \and
    Duy Nhat Phan$^{1}$ \\
    \and
    Andrew Rice$^{1}$ \\
}

\date{}
\begin{document}
\maketitle
\footnotetext[1]{University of Dayton Research Institute, University of Dayton, 300 College Park, Dayton, 45469, Ohio, USA; emails: \{patrick.hytla,duynhat.phan,andrew.rice\}@udri.udayton.edu}
\footnotetext[2]{Department of Mathematics and Statistics, Oakland University, Rochester, MI 48309, USA; email: nttran@oakland.edu}
\footnotetext[3]{Author names are listed in alphabetical order.}

\begin{abstract}{\small 
 Matrix completion is fundamental for predicting missing data with a wide range of applications in personalized healthcare, e-commerce, recommendation systems, and social network analysis. Traditional matrix completion approaches typically assume centralized data storage, which raises challenges in terms of computational efficiency, scalability, and user privacy.  In this paper, we address the problem of federated matrix completion, focusing on scenarios where user-specific data is distributed across multiple clients, and privacy constraints are uncompromising. Federated learning provides a promising framework to address these challenges by enabling collaborative learning across distributed datasets without sharing raw data. We propose \texttt{FedMC-ADMM} for solving federated matrix completion problems, a novel algorithmic framework that combines the Alternating Direction Method of Multipliers with a randomized block-coordinate strategy and alternating proximal gradient steps. Unlike existing federated approaches, \texttt{FedMC-ADMM} effectively handles multi-block nonconvex and nonsmooth optimization problems, allowing efficient computation while preserving user privacy. We analyze the theoretical properties of our algorithm, demonstrating subsequential convergence and establishing a convergence rate of $\mathcal{O}(K^{-1/2})$, leading to a communication complexity of $\mathcal{O}(\epsilon^{-2})$ for reaching an $\epsilon$-stationary point. This work is the first to establish these theoretical guarantees for federated matrix completion in the presence of multi-block variables. To validate our approach, we conduct extensive experiments on real-world datasets, including MovieLens 1M, 10M, and Netflix. The results demonstrate that \texttt{FedMC-ADMM} outperforms existing methods in terms of convergence speed and testing accuracy.}
\end{abstract} 


\section{Introduction}
Incomplete data is a common challenge across numerous domains, from personalized healthcare \cite{tran2021recommender} and e-commerce \cite{sivapalan2014recommender,ricci2010introduction,Koren2009} to content streaming platforms \cite{gomez2015netflix,kang2018self}, sensor networks \cite{Biswas2006}, social network analysis \cite{Kim2011}, and image processing \cite{Lui2013}. For instance, in recommendation systems, user preferences for items are  partially observed, leading to what is known as the matrix completion (MC) problem \cite{Koren2009}. Addressing this challenge is essential for enabling accurate predictions and informed decision-making across these diverse applications.

In a centralized setting, given an incomplete data matrix $M \in \mathbb{R}^{m \times n}$ with missing entries, the goal of the MC problem is to find factors $U \in \mathbb{R}^{m \times r}$ and $V \in \mathbb{R}^{r \times n}$ such that $M \approx UV$. This task is commonly formulated as the following optimization problem
\begin{equation*}\label{MF}
    \min_{U \in \mathcal{U}, V \in \mathcal{V}}\quad  \frac{1}{2} \|\mathcal{P}(M - UV)\|_F^2 + R_1(U) + R_2(V),
\end{equation*}
where $\mathcal{P}:\mathbb{R}^{m\times n}\to \mathbb{R}^{m\times n}$ is the linear operator defined by $\mathcal{P}(Z)_{ij} = Z_{ij}$, for $Z\in \mathbb{R}^{m\times n}$ if $M_{ij}$ is observed and $0$ otherwise. The sets $\mathcal{U}$ and $\mathcal{V}$ are constraints, while $R_1$ and $R_2$ are regularizers that impose specific structures on variables $U$ and $V$, respectively.

However, centralized data collection and processing face several challenges. Traditional MC methods demand extensive storage for observed data and substantial computational resources, particularly when working with large-scale datasets, leading to high resource costs. Moreover, these approaches often rely on single machines or small clusters, limiting scalability and efficiency in handling complex datasets. Privacy concerns further complicate centralized methods, especially in applications such as recommendation systems, where protecting sensitive user data is paramount \cite{canny2002collaborative}. Clients are often reluctant to share personal information, yet traditional matrix completion approaches assume full access to user activity data, creating a significant obstacle to adoption in privacy-sensitive contexts.

To address these challenges, we investigate the MC problem in the federated learning (FL) format. In this setting, we consider \( p \) clients, where the \( i  \)-th client holds a private data matrix \( M_i \in \mathbb{R}^{m_i \times n} \), with no overlap between the matrices held by different clients. A central server coordinates the clients to collaboratively solve the following optimization problem without requiring them to share raw data:
\begin{equation}\label{model}
    \begin{aligned}
    \min_{\substack{\mbfU, V}} \quad & F(\mbfU, V) := \frac{1}{p} \sum_{i=1}^p F_i(M_i, U_i V) + R(V), \\
    \text{subject to} \quad & U_i \in \mathcal{U}_i \text{ for } i \in \{1, \ldots, p\}, \quad V \in \mathcal{V},
    \end{aligned}
\end{equation}
where \( F_i(M_i, U_i V) \) represents the local loss function of client \( i \), \( R(V) \) is a regularizer imposing structure on the global variable \( V \), and \( \mathcal{U}_i \subseteq \mathbb{R}^{m_i \times r} \) and \( \mathcal{V} \subseteq \mathbb{R}^{r \times n} \) are constraints. Here, \( \mathbf{U} = [U_1; \ldots; U_p] \) contains client-specific variables, while \( V \) is a global variable shared across all clients. For example, in federated recommendation systems, each client $i$ may represent a different user group, where \( M_i \) is the incomplete user-item rating matrix, \( U_i \) is the user profiles for client \( i \), and \( V \) is the shared item attributes, see Figure \ref{rec}.

\begin{figure*}[!htpb] 
\vspace{-1ex}
\begin{center}
\includegraphics[width=0.9\linewidth]{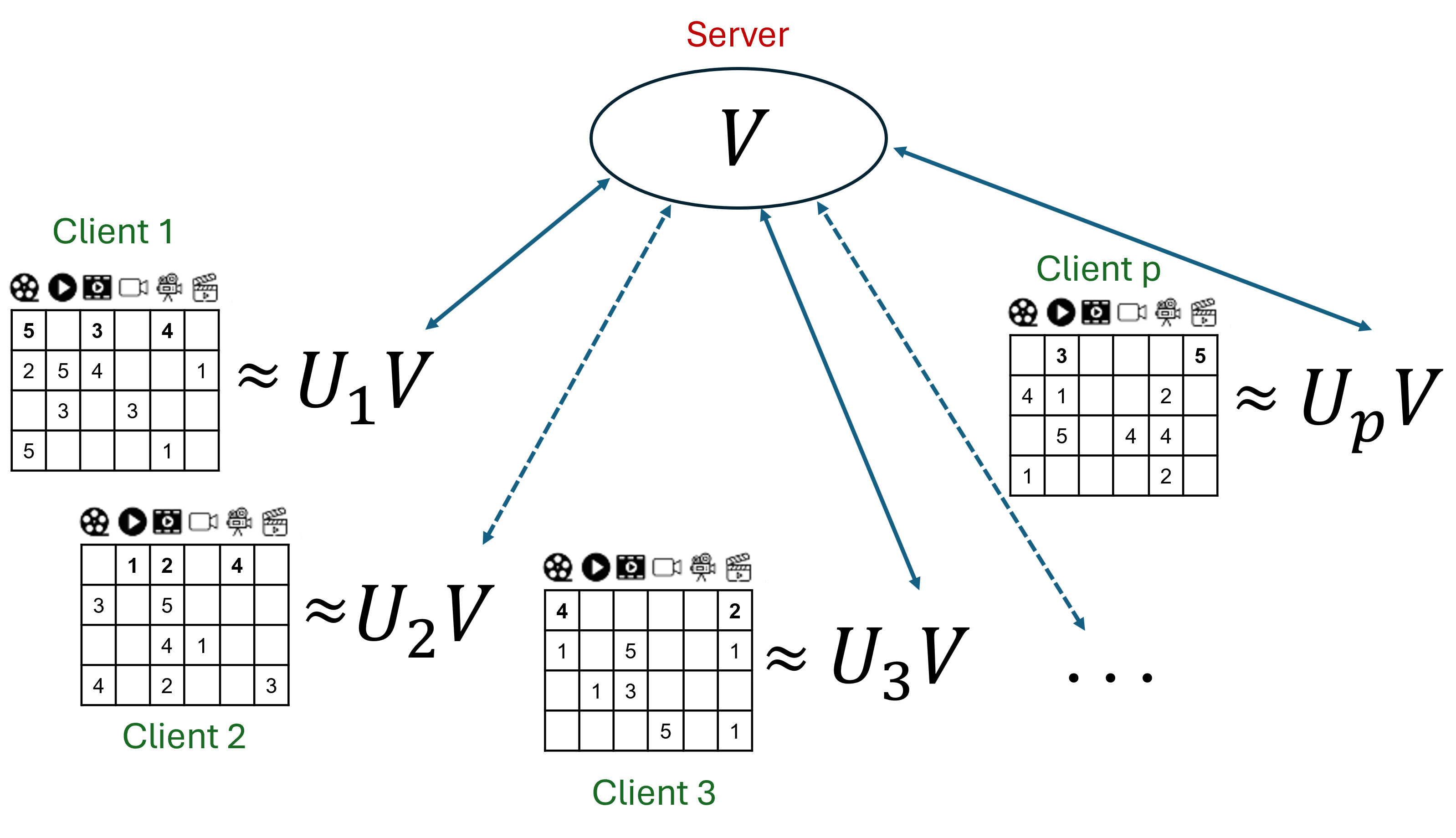}
\caption{The typical scenario of federated recommendation systems. Each client has their own user profiles and shares the same items.} \label{rec}
\end{center}
\vspace{-1ex}
\end{figure*}

Solving the federated MC problem \eqref{model} is challenging due to its nonconvexity, non-smoothness, and multi-block structure, particularly when balancing communication efficiency and privacy preservation. Recent advances in federated algorithms based on the Alternating Direction Method of Multipliers (ADMM) have demonstrated efficiency and scalability, achieving state-of-the-art performance \cite{huang2019dp,ding2019stochastic,zhou2023federated,phan4822244federated}. Building on this, we propose \texttt{FedMC-ADMM} (Federated Matrix Completion with randomized ADMM and proximal steps), a novel framework to address the MC problem in FL. Our approach integrates a revised ADMM with randomized block-coordinate updates and proximal steps, enabling efficient handling of multi-block variables, nonsmooth regularizers, and inexact subproblem solutions at the clients. This is the first method to incorporate randomized block-coordinate ADMM with proximal steps for the federated MC problem.

From a theoretical perspective, under mild assumptions, we establish subsequential convergences of the algorithm, proving that every limit point of the generated sequence is almost surely a stationary point of \eqref{model}. Furthermore, our analysis demonstrates that \texttt{FedMC-ADMM} achieves a convergence rate of \( \mathcal{O}(K^{-1/2}) \) near a stationary point and requires at most \( \mathcal{O}(\epsilon^{-2}) \) communication rounds to reach an \( \epsilon \)-stationary point. 
These results are particularly challenging to achieve for the nonconvex federated MC problem due to the interplay of nonconvexity, non-smoothness, and multi-block structures.

Finally, we validate the practical efficacy of \texttt{FedMC-ADMM} through numerical experiments on real-world datasets, including MovieLens 1M, MovieLens 10M, and Netflix. The results highlight the algorithm's capability to provide privacy-preserving, efficient, and scalable solutions for the MC problem in FL settings.

\section{Related Work}

To date, existing FL algorithms have yet to address the general federated MC problem \eqref{model}, which involves the challenges of nonconvexity, non-smoothness, and multi-block structures considered in this work. Accordingly, we discuss related literature on matrix completion, federated matrix factorization, and federated learning.

Traditional matrix completion methods are primarily designed for centralized settings, assuming full access to the data. These classical approaches rely on low-rank matrix factorization to impute missing entries \cite{Koren2009,Mazumder2010,Hien2023InertialBlockMM}. However, privacy concerns in matrix completion have recently gained attention. For instance, \cite{Chien2021PrivateALS} introduced private alternating least squares to achieve practical matrix completion with tighter convergence rates, while \cite{Wang2023Differential} proposed a differentially private approach based on low-rank matrix factorization. In the context of federated MC problems,  \cite{hegedHus2019decentralized,chai2020secure} explored methods using gradient sharing (GS), where a central server aggregates client gradients for gradient descent updates. Despite their potential, GS methods are known to suffer from slow convergence, requiring many communication rounds to achieve a quality model, as highlighted in \cite{mcmahan2017communication}. To address specific applications such as clustering, \cite{wang2022federated} proposed federated matrix factorization techniques based on model averaging and GS. It is important to note that these existing federated MC methods exclusively
 focus on a special case of \eqref{model}, where the regularization functions are either zero or smooth. As a result, they fail to tackle the broader challenges posed by the general federated MC problem, particularly those involving nonconvexity, non-smoothness, and multi-block structures.

Federated learning has made significant progress in addressing communication efficiency, statistical heterogeneity, and privacy preservation \cite{mcmahan2017communication,jain2021differentially,shen2023share}. Early methods such as FedAvg \citep{mcmahan2017communication} focused on averaging model updates across clients, while more recent approaches, such as FedProx \citep{li2020federated}, introduced a proximal penalty term to stabilize training in heterogeneous environments. SCAFFOLD \citep{karimireddy2020scaffold} further improved convergence rates by reducing variance in client updates. Despite these advances, these methods are primarily designed for training federated optimization problems with a single block of variables and are not suitable for solving \eqref{model}, which involves multi-block and nonsmooth regularizers.

ADMM is a well-established optimization technique \cite{glowinski1975approximation,gabay1976dual} with  applicability to nonconvex settings \cite{boct2020proximal,wang2019global}. In the context of federated learning, several approaches have been proposed: DP-ADMM \cite{huang2019dp}, which incorporates differential privacy through Gaussian noise; PS-ADMM \cite{ding2019stochastic}, a stochastic ADMM method with privacy guarantees; fedADMM \cite{zhou2023federated}, which combines ADMM with a cyclic rule for smooth federated optimization problems; and a randomized ADMM approach \cite{phan4822244federated} designed to improve efficiency in training neural networks within FL settings. These methods, however, do not generalize to the multi-block, nonconvex, nonsmooth MC problem considered in this work.

Building on these distributed matrix factorization techniques, many approaches have focused on parallelizing sequential algorithms like stochastic gradient descent or alternating least squares such as MapReduce \cite{LSMF_DSGD_2011,DMF_MRBJ_2013,FDSGD_MF_2014} or parameter servers \cite{Factorbird_2014}. Similar strategies have been applied to non-negative matrix factorization \cite{DNMF_WSDD_2010,DNMF_HALS_2017}. While decentralized methods \cite{wai2015consensus,pmlr-v70-hong17a,PPPIR_DMF_2019} remove the need for a central server, they often rely on shared memory or specific data partitions. These assumptions make decentralized methods incompatible with federated settings, where privacy and communication constraints are critical. Consequently, these existing approaches fail to address the unique challenges posed by federated matrix completion, which requires balancing privacy preservation, communication efficiency, and optimization across distributed clients.

In summary, our algorithm distinguishes itself from existing federated matrix completion and ADMM-based approaches by combining ADMM with randomized block-coordinate updates and proximal steps for MC in FL. Unlike prior methods, it addresses the challenges of multi-block, nonconvex, and nonsmooth optimization problems while achieving state-of-the-art convergence rates. Its scalability and privacy-preserving design also make it highly applicable to real-world federated recommendation systems.

\section{Notations and Preliminaries}\label{sec:preliminary}

The inner product in a finite-dimensional real linear space \(\mathbf{E}\) is represented by \(\scal{\cdot}{\cdot}\), and its corresponding induced norm \(\|\cdot\|\) is defined for any \(x \in \mathbf{E}\) as \(\|x\| := \sqrt{\scal{x}{x}}\). We define \(\mathbb{R}_+ := \{r \in \mathbb{R} : r \geq 0\}\) as the set of non-negative real numbers and use the notation \([p]\) to represent the set of integers \(\{1, 2, \dots, p\}\). For matrices \(U_i \in \mathbb{R}^{m_i \times r}\), \(i \in [p]\), we employ the notation \([U_1;U_2;\ldots;U_p]\) to represent their vertical concatenation, expressed as
\[
[U_1;U_2;\ldots;U_p] := \begin{bmatrix}
    U_1 \\ U_2 \\ \vdots \\ U_p
    \end{bmatrix} \in \mathbb{R}^{m \times r},
\]
where \(m = \sum_{i=1}^p m_i\). Given a nonempty closed set \(\mathcal{C} \subset \mathbf{E}\), we define the distance from a point \(x \in \mathbf{E}\) to \(\mathcal{C}\) as
\[
\dist(x, \mathcal{C}) := \inf_{y \in \mathcal{C}} \|x - y\|.
\]

An extended-real-valued function \(g : \mathbf{E} \to \mathbb{R} \cup \{+\infty\}\) is said to be proper if its domain \({\rm dom}\, g := \{x \in \mathbf{E} : g(x) < +\infty\}\) is nonempty. We call a function \(g\) lower semicontinuous if for every point \(x^* \in \mathbf{E}\), it satisfies
\[
g(x^*) \leq \liminf_{x \to x^*} g(x).
\]

Let \(\rho\) be a real number. A function \(g\) is called \(\rho\)-convex if \(g - \frac{\rho}{2} \|\cdot\|^2\) is convex. Equivalently, \(g\) is \(\rho\)-convex if, for any \(x, y \in \mathbf{E}\) and \(\lambda \in [0, 1]\),
\[
g((1 - \lambda)x + \lambda y) \leq (1 - \lambda)g(x) + \lambda g(y) - \frac{\rho}{2} \lambda(1 - \lambda) \|x - y\|^2.
\]
A function is convex if it is \(0\)-convex. If \(g\) is \(\rho\)-convex with \(\rho > 0\), it is referred to as \(\rho\)-strongly convex.

\begin{definition}[Fr\'echet and limiting subdifferentials, {\citep[Definition 8.3]{VariationalAnalysis}}] \label{def:dd}
Consider a proper, lower semicontinuous function \(f : \mathbf{E} \to \mathbb{R} \cup \{+\infty\}\).
\begin{enumerate}[label={\rm(\alph*)}]
    \item For a point \(x \in {\rm dom}\,f\), the \emph{Fr\'echet subdifferential} of \(f\) at \(x\), denoted by \(\hat{\partial}f(x)\), consists of all vectors \(v \in \mathbf{E}\) satisfying
    \[
    \liminf_{y \neq x, y \to x} \dfrac{f(y) - f(x) - \scal{v}{y - x}}{\|y - x\|}  \geq 0.
    \]
    We set \(\hat{\partial}f(x) = \emptyset\) for any \(x \notin {\rm dom}\,f\).
    
    \item For a point \(x \in {\rm dom}\,f\), the \emph{limiting subdifferential} of \(f\) at \(x\), denoted by \(\partial f(x)\), is defined as
    \[
    \partial f(x) := \big\{v \in \mathbf{E} : \exists\, x^k \to x, f(x^k) \to f(x), v^k \in \hat{\partial}f(x^k), v^k \to v\big\}.
    \]
\end{enumerate}
\end{definition}

For a convex function \(f\), the Fr\'echet and limiting subdifferentials coincide with the classical convex subdifferential, given by
\[
\partial f(x) = \big\{v\in \mathbf{E}: f(y) \geq f(x) + \scal{v}{y - x} \text{ for all } y \in \mathbf{E}\big\}.
\]

\begin{definition}[\(L\)-Lipschitz]
A mapping \(T : D \subset \mathbf{E} \to \mathbb{F}\) is said to be \(L\)-Lipschitz on \(D\), with constant \(L \geq 0\), if 
\[
\|T(x) - T(y)\| \leq L \|x - y\| \quad \text{for all} \quad x, y \in D.
\]
\end{definition}

\begin{definition}[\(L\)-Smooth]
A function \(f : \mathbf{E} \to \mathbb{R}\) is said to be \(L\)-smooth on \(D\) if it is differentiable everywhere on \(D\) and its gradient \(\nabla f\) is \(L\)-Lipschitz on \(D\).
\end{definition}

We now recall some fundamental properties that will be used throughout this work.

\begin{lemma}\label{convexity-smoothness}
Let \(g\) and \(h\) be proper and lower semicontinuous functions. Then, the following statements hold.
\begin{enumerate}[label={\rm(\alph*)}]
    \item If \(g\) attains a local minimum at \(\bar{x} \in {\rm dom}\,g\), then \(0 \in \partial g(\bar{x})\).
    
    \item If \(f = g + h\) and \(h\) is continuously differentiable in a neighborhood of \(\bar{x}\), then \(\partial f(\bar{x}) = \partial g(\bar{x}) + \nabla h(\bar{x})\).
    
    \item If \(g\) is \(\alpha\)-strongly convex and \(y = \argmin_z \{g(z) + \frac{\rho}{2} \|z - x\|^2\}\), then for all \(x \in \mathbf{E}\): \label{cvx:min}
    \[
    g(x) \geq g(y) + \frac{2\rho + \alpha}{2}\|x - y\|^2.
    \] 
    
    \item If \(g\) is \(L\)-smooth, then for all \(x, y \in \mathbf{E}\):
    \[
    |g(x) - g(y) - \scal{\nabla g(y)}{x - y}| \leq \frac{L}{2}\|x - y\|^2.
    \]
\end{enumerate}
\end{lemma}

\begin{proof}
(a) See \cite[Theorem 8.15]{VariationalAnalysis}. 
(b) See \cite[Exercise 8.8]{VariationalAnalysis}. 
(c) See \cite[Theorem 6.39]{beck2017first}. 
(d) See \cite[Lemma 1.2.3]{nesterov2018lectures}.
\end{proof}

We develop a probabilistic framework, following the approach in \cite{Cannelli2019Asynchronous}, to analyze the behavior of Algorithm \ref{FedLMFmm}. For each communication round \(k \in \mathbb{N}\), let \(S_k\) represent a subset of clients \([p]\). Define \(\Omega\) as the sample space containing all possible sequences \(\omega = \{w_k\}_{k=0}^{\infty}\), where \(\omega_k = S_k\). We construct a filtration \(\{\mathcal{F}_k\}\) on \(\Omega\) as follows. For any \(k \geq 1\) and given sequence \((S_0, S_1, \ldots, S_{k-1})\), we define the cylinder set
\[
C(S_0, S_1, \ldots, S_{k-1}) := \{\omega \in \Omega : \omega_0 = S_0, \omega_1 = S_1, \ldots, \omega_{k-1} = S_{k-1}\}.
\]
Let \(C^k\) denote the collection of all such cylinder sets \(C(S_0, S_1, \ldots, S_{k-1})\). We then define \(\mathcal{F}_k := \sigma(C^k)\) as the \(\sigma\)-algebra generated by \(C^k\), and \(\mathcal{F} := \sigma\left(\bigcup_{k=1}^{\infty} C^k\right)\). The sequence \(\{\mathcal{F}_k\}\) forms a filtration satisfying \(\mathcal{F}_k \subset \mathcal{F}_{k+1} \subset \mathcal{F}\) for all \(k \geq 1\). The \(\sigma\)-algebra \(\mathcal{F}\) is equipped with a probability measure \(P\), creating the probability space \((\Omega, \mathcal{F}, P)\). We denote by \(\E_k[\cdot]\) the conditional expectation \(\E[\cdot | \mathcal{F}_k]\) given \(\mathcal{F}_k\), and use \(\E[\cdot]\) for the total expectation.

The following result provides the theoretical foundation for our convergence analysis.

\begin{lemma}[Supermartingale Convergence, {\citep[Theorem 1]{robbins1971convergence}}]
\label{supermartingale}
Let \(\{Y_k\}, \{Z_k\}, \{W_k\}\) be three sequences of random variables, and \(\{\mathcal{F}_k\}\) be a filtration such that \(\mathcal{F}_k \subset \mathcal{F}_{k+1}\) for all \(k\). Assume that:
\begin{enumerate}[label={\rm(\alph*)}]
    \item The random variables \(Y_k, Z_k\), and \(W_k\) are nonnegative and \(\mathcal{F}_k\)-measurable for each \(k\).
    
    \item For each \(k\), the conditional expectation satisfies \(\E_k[Y_{k+1}] \leq Y_k - Z_k + W_k\).
    
    \item Almost surely, \(\sum_{k=0}^\infty W_k < +\infty\).
\end{enumerate}
Then, almost surely, \(\sum_{k=0}^\infty Z_k < +\infty\) and the sequence \(\{Y_k\}\) converges to a nonnegative random variable.
\end{lemma}

\section{Algorithm Description and Convergence Analysis}\label{sec:algo}

\subsection{Algorithm Description}

As discussed in the Introduction section, the main objective of our paper is designing an algorithm to solve the optimization problem~\eqref{model}. Each loss function $F_i(M_i, U_iV)$, $i\in [p]$ is a function of $(U_i,V)\in \mathbb{R}^{m_i\times r}\times \mathbb{R}^{r\times n}$, as $M_i$ is a known $m\times n$ matrix.  To address further  the matrix completion problem under the federated learning (FL) framework, we may assume that the loss function $F_i$ for client \(i\) takes the form
\begin{equation}\label{eq:F_i}
F_i(M_i, U_i V) = f_i(U_i, V) + R_i(U_i),
\end{equation}
where \(f_i:\mathbb{R}^{m_i \times r} \times \mathbb{R}^{r \times n} \to \mathbb{R}\) is a continuously differentiable function and \(R_i:\mathbb{R}^{m_i\times r}\to \mathbb{R}\cup\{+\infty\}\) is a local regularizer for client \(i\) that is a proper lower semi-continuous function. To reformulate problem \eqref{model} for decentralized optimization, we introduce auxiliary variables \(W_i = V\) for \(i \in[p]\) and rewrite  it  as
\begin{equation}\label{admmmodel}
    \begin{aligned}
        \min_{\substack{U_i \in \mathcal{U}_i, W_i \in \mathbb{R}^{r \times n}, V \in \mathcal{V}}} \quad & \frac{1}{p} \sum_{i=1}^p \left[f_i(U_i, W_i) + R_i(U_i)\right] + R(V) \\
        \text{subject to} \quad & W_i - V = 0 \quad \forall i \in [p].
    \end{aligned}
\end{equation}
The {\em augmented Lagrangian function} corresponding to this  problem is defined by
\begin{equation}\label{def:AugL}
\mathcal{L}(\mathbf{U}, \mbfW, V, \mathbf{Y}) := \sum_{i=1}^p \mathcal{L}_i(U_i, W_i, V, Y_i) + R(V),
\end{equation}
where \(\mathbf{U} := [U_1; \ldots; U_p]\), \(\mbfW := [W_1; \ldots; W_p]\), \(\mathbf{Y} := [Y_1; \ldots; Y_p]\), and 
\begin{equation}\label{def:AL}
\mathcal{L}_i(U_i, W_i, V, Y_i) := \frac{1}{p} \left[f_i(U_i, W_i) + R_i(U_i)\right] + \langle Y_i, W_i - V \rangle + \frac{\beta}{2} \|W_i - V\|^2
\end{equation}
with some penalty parameter \(\beta > 0\).

The classical ADMM scheme \cite{glowinski1975approximation, gabay1976dual} for solving \eqref{admmmodel} iteratively updates \((\mathbf{U}, \mbfW)\), \(V\), and \(\mathbf{Y}\) as follows
\begin{align}
    (U_i^{k+1}, W_i^{k+1}) & \in \argmin_{U_i \in \mathcal{U}_i, W_i} \mathcal{L}_i(U_i, W_i, V^k, Y_i^k), \quad \forall i \in [p], \label{subprob1} \\
    V^{k+1} & \in \argmin_{V \in \mathcal{V}} \sum_{i=1}^p \left[\langle Y_i^k, W_i^{k+1} - V \rangle + \frac{\beta}{2} \|W_i^{k+1} - V\|^2\right] + R(V), \label{subprob2} \\
    Y_i^{k+1} & = Y_i^k + \beta (W_i^{k+1} - V^{k+1}), \quad \forall i \in [p]. \label{subprob3}
\end{align}
This method is highly suitable for Federated Learning, as  each client can keep their data private for its own sake without sharing it with the server and any other clients.  Indeed, after the local training \eqref{subprob1},  the client $i$ keeps $U_i^{k+1}$ and sends the pair $(W_i^{k+1}, Y_i^{k})$ to the server to solve problem~\eqref{subprob2} and update $V^{k+1}$. It is not possible for the server or any other clients to approximate the data $M_i$ by using only $(W_i^{k+1}, Y_i^{k})$ without $U_i^{k+1}$. The whole process does not violate the data privacy of each client at all. However, this classical ADMM scheme faces significant challenges when applied to the federated MC problem \eqref{admmmodel}. One of the primary limitations lies in its lack of theoretical convergence guarantees under nonconvex settings. Specifically, both \(F_i\) and \( R \) are often nonsmooth, whereas classical ADMM generally requires at least one of these components to be smooth to guarantee theoretical convergence. Furthermore, the linear operator \( A = [-I_{rn}; \ldots; -I_{rn}] \), which enforces the equality constraints \( W_i = V \), is not surjective. This violates key conditions for convergence guarantees in nonconvex ADMM as established in \cite{boct2020proximal, wang2019global, hien2022inertial}. Another major drawback is the computational difficulty associated with solving the subproblem \eqref{subprob1} for updating \((U_i, W_i)\) simultaneously. These subproblems are inherently nonconvex and nonsmooth, making them computationally expensive and challenging to solve efficiently, particularly for large-scale problems with federated constraints. Communication inefficiency is also a significant limitation of classical ADMM in FL. Frequent communication between local clients and the central server to update 
\((U_i, W_i)\), \(V\), and \(\mathbf{Y}\) at every iteration incurs high overhead, particularly as the number of clients grows. The algorithm is also sensitive to slower clients, which can delay progress and degrade system performance. Moreover, classical ADMM lacks mechanisms to handle clients with limited transmission resources or computational capacity, further exacerbating scalability issues. Its rigid requirement for all clients to update their local models at each iteration, irrespective of constraints, reduces efficiency. By contrast, updating only a subset of clients can achieve similar global performance while alleviating these challenges.

To address the above challenges, we propose a novel algorithm, \texttt{FedMC-ADMM}, which is specifically designed for the federated matrix completion problem. This algorithm employs a randomized block-coordinate strategy, which allows for partial client participation in each iteration, thereby significantly reducing communication costs and mitigating the impact of lower clients. By incorporating alternating proximal gradient updates for local subproblems, \texttt{FedMC-ADMM} ensures closed-form solutions in many practical cases, enhancing computational efficiency. Furthermore, the update order of dual variables \(\mathbf{Y}\) and the global variable \(V\) is modified, which guarantees convergence properties and stability. These modifications collectively enable \texttt{FedMC-ADMM} to provide a robust, scalable, and efficient solution for the federated MC problem.

Specifically, instead of solving \eqref{subprob1} in both blocks \(U_i\) and \(W_i\) at the same time that could be impractical in many situations especially when the function $R_i$ is not strongly convex, we fix \(W_i = W_i^k\) first and update \(U_i\) through \(N\) proximal gradient steps. For \(l = 1, \ldots, N\), we linearize \(f_i\) with respect to \(U_i\) at \(U_i^{k,l-1}\), and update \(U_i^{k,l}\) via a proximal step as follows
\begin{equation}\label{125a}
    U_i^{k,l} \in \argmin_{U_i \in \mathcal{U}_i}\; \langle \nabla_U f_i(U_i^{k,l-1}, W_i^k), U_i \rangle + \frac{L_{W_i^k}}{2} \|U_i - U_i^{k,l-1}\|^2 + R_i(U_i),
\end{equation}
where \(U_i^{k,0} = U_i^k\) and $L_{W_i^k}$ is the Lipschitz constant  of  $\nabla_U f_i(\cdot,W_i^k)$ from our later standing assumption~\eqref{average-smoothness}. In many popular cases of regularizers $R_i$, $U_i^{k,l}$ can have closed forms. After completing \(N\) steps, \(U_i^{k+1}\) is set as \(U_i^{k,N}\).

Next, we fix \(U_i = U_i^{k+1}\) in \eqref{subprob1} and update \(W_i\) through a similar process by linearizing \(f_i\) with respect to \(W_i\) at \(W_i^{k,l-1}\), $l=1,\ldots,N$ with $W_i^{k,0}=W_i^k$ and updating \(W_i^{k,l}\) as
\begin{equation*}\label{125b}
    W_i^{k,l} = \argmin_{W_i}\; \langle \nabla_V f_i(U_i^{k+1}, W_i^{k,l-1}) / p + Y_i^k, W_i \rangle + \frac{L_{U_i^{k+1}} / p}{2} \|W_i - W_i^{k,l-1}\|^2 + \frac{\beta}{2} \|W_i - V^k\|^2,
\end{equation*}
where $L_{U_i^{k+1}}$ the Lipschitz constant   of the function $f_i(U_i^{k+1}, \cdot)$. Notably, this subproblem admits a closed-form solution
\begin{equation}\label{eq:W_i}
    W_i^{k,l} = \frac{L_{U_i^{k+1}} / p \cdot W_i^{k,l-1} + \beta V^k -\nabla_V f_i(U_i^{k+1}, W_i^{k,l-1})/p - Y_i^k}{L_{U_i^{k+1}} / p + \beta}.
\end{equation}
After completing \(N\) steps, \(W_i^{k+1}\) is set as \(W_i^{k,N}\).

\begin{remark}[Another way of updating $W_i^{k+1}$] \label{RemP}{\rm It is possible to update $W_i$ by linearization from \eqref{subprob1} by solving another optimization problem
\begin{equation*}\label{eq:Win}
W_i^{k+1} \in  \argmin_{W_i}\; \langle \nabla_V f_i(U_i^{k+1}, W_i^{k}) / p + Y_i^k, W_i \rangle +  \frac{\beta}{2} \|W_i - V^k\|^2,
\end{equation*}
which gives us that 
 \begin{equation}\label{eq:Wik}
0=\frac{1}{p}\nabla_{V} f_i(U_i^{k+1}, W^{k+1}_i)+Y^k_i+\beta (W^{k+1}_i-V^k).
 \end{equation}
 This update is different from the one in \eqref{eq:W_i} and seems to be less expensive. Let us analyze a bit more these two updates for the  ferderated matrix completion problem in the most popular case $f(U_i,W_i)=\frac{1}{2}\|\mathcal{P}_{\Omega_i}(M_i-U_iW_i)\|^2$, where $\Omega_i=\{(t,j)\in[m_i]\times[n]|\, (M_i)_{tj}\neq 0\}$, which is the observation index from $M_i$ and 
 $\mathcal{P}_{\Omega_i}:\mathbb{R}^{m_i\times n}\to \mathbb{R}^{m_i\times n}$ is the projection operator defined by $[\mathcal{P}_{\Omega_i}(Z)]_{tj}=Z_{tj}$ if ${(t,j)}\in \Omega_i$ and $0$ otherwise. In this case, formula \eqref{eq:Wik} turns into  
\[
0=\left(U_i^{k+1}\right)^T\mathcal{P}_{\Omega_i}(U_i^{k+1} W^{k+1}_i-M_i)/p+Y^k_i+\beta (W^{k+1}_i-V^k).
\]
It is not clear how to compute $W^{k+1}_i$ directly from this equation. But, by looking into each column $j\in[n]$ of $W^{k+1}_i$, we observe that
\[
0=\left[\sum_{t:(t,j)\in\Omega_i}\left(\left(U_i^{k+1}\right)_{t:}\right)^T\left(U_i^{k+1}\right)_{t:}(W^{k+1}_i)_{:j}/p+\beta(W^{k+1}_i)_{:j}\right]-\left(\left(U_i^{k+1}\right)^T\mathcal P_{\Omega_i}(M_i)/p+\beta V^k-Y^k_i\right)_{:j},
\]
which gives us that 
 \[
 (W_i^{k+1})_{:j}=\left[\sum_{t:(t,j)\in\Omega_i}\left(\left(U_i^{k+1}\right)_{t:}\right)^T\left(U_i^{k+1}\right)_{t:}/p+\beta\mathbb{I}_{r\times r}\right]^{-1}\left(\left(U_i^{k+1}\right)^T\mathcal P_{\Omega_i}(M_i)/p+\beta V^k-Y^k_i\right)_{:j},
 \]
 where $(W_i^{k+1})_{:j}$ denotes the $j$-th column of $W_i^{k+1}$, while $(U_i^{k+1})_{t:}$ denotes $t$-th row of $U_i^{k+1}$. As $\Omega_i$ is usually a small subset of $[m_i]\times [n]$ and $r$ is also a small number in the federated matrix completion problem, computing the inverse matrix in the above formula is quite simple. However, we have to compute totally $n$ different inverse matrices to update the whole $W_i^{k+1}$. As $n$ is usually a very large number (up to $17,770$ for the Netflix data; see Section~5), this proceed of updating $W_i^{k+1}$ turns out to be very expensive. 
 
 On the other hand, our update $W^{k,l}_i$ in \eqref{eq:W_i} is quite simple at every step $l=1, \ldots, N$ without any inversion and the number $N$ is usually chosen to be small in our analysis (up to $30$ in very large datasets). The only extra step that we have to compute is the Lipschitz constant $L_{U_i^{k+1}}$  of $\nabla_V f_i(U_i^{k+1}, \cdot)$ in \eqref{eq:W_i}. In this case, for any $W,V\in \mathcal{V}$, with the typical Frobenius norm $\|\cdot\|_F$ for matrices  note that
 \begin{equation}
 \begin{aligned}
 \|\nabla_{V} f_i(U_i^{k+1}, W)-\nabla_{V} f_i(U_i^{k+1}, V)\|_F&=\|(U_i^{k+1})^T\mathcal{P}_{\Omega_i}(U_i^{k+1}(W-V))\|_F\\
 &=\left\|\left[\sum_{t:(t,j)\in\Omega_i}\left(\left(U_i^{k+1}\right)_{t:}\right)^T\left(U_i^{k+1}\right)_{t:}(W-V)_{:j}\right]_{j=1,\ldots,n}\right\|_F\\
&\le \sum_{j}^n\sqrt{\left\|\sum_{t:(t,j)\in\Omega_i}\left(\left(U_i^{k+1}\right)_{t:}\right)^T\left(U_i^{k+1}\right)_{t:}\right\|\cdot\|(W-V)_{:j}\|^2}\\
&\le \max_{1\le j\le n}\sqrt{\left\|\sum_{t:(t,j)\in\Omega_i}\left(\left(U_i^{k+1}\right)_{t:}\right)^T\left(U_i^{k+1}\right)_{t:}\right\|}\cdot \|W-V\|_{F}\\
&\le \|(U_i^{k+1})^TU_i^{k+1}\|_F\cdot\|W-V\|_{F},
 \end{aligned}
 \end{equation}
 which means that $L_{U_i^{k+1}}$ can be chosen as $\|(U_i^{k+1})^TU_i^{k+1}\|_F$, which is also easy to compute .
 \hfill$\triangle$
}
\end{remark}

As discussed after \eqref{subprob3}, the classical ADMM scheme for the federated matrix completion problem may face lots of challenges in convergence guarantee, as the model does not satisfy the traditional conditions for the ADMM iteration to be converged. We come up with the idea of switching the update order of \(V\) in \eqref{subprob2} and \(\mathbf{Y}\) in \eqref{subprob3}. Surprisingly, this swap allows us to  obtain reasonable convergence and complexity for our algorithm proved later in this section.  After updating \(U_i^{k+1},W_i^{k+1}\) in \eqref{125a} and \eqref{eq:W_i}, each client \(i\) updates its dual variable \(Y_i^{k+1}\) as
\begin{equation}\label{eq:Y_i}
    Y_i^{k+1} = Y_i^k + \beta (W_i^{k+1} - V^k).
\end{equation}


It remains to update $V_{k+1}$. Instead of requiring updates from all local clients at every communication, we employ a randomized block-coordinate strategy that can help the server proceed  quickly. In this approach, only a subset of clients, \(S_k \subseteq [p]\), perform local updates and send their updated models \((W_i^{k+1}, Y_i^{k+1})\) to the server.  Other clients  \(i \notin S_k\) can rest in this round, i.e., their local variables remain unchanged: \(U_i^{k+1} \leftarrow U_i^k\), \(W_i^{k+1} \leftarrow W_i^k\), and \(Y_i^{k+1} \leftarrow Y_i^k\).   Finally, the server updates the global variable \(V^{k+1}\) by solving
\begin{equation*}
    V^{k+1} \in \argmin_{V \in \mathcal{V}} \sum_{i=1}^p \left[\langle Y_i^{k+1}, W_i^{k+1} - V \rangle + \frac{\beta}{2} \|W_i^{k+1} - V\|^2\right] + R(V).
\end{equation*}
To make sure the algorithm is convergent, the choice of $S_k$ at  round $k$ has to follow a uniform randomness such that each client $i$, $i\in [p]$ is chosen to participate with probability at least $p_i>0$. The probability $p_i$ for each client could be different and probably depend on various initial situations or negotiations between the clients and the training center/server. This algorithm design reduces communication overhead, enables efficient handling of (possibly nonconvex) subproblems, and mitigates the impact of stragglers by allowing partial client updates in each communication. We summarize them in Algorithm \ref{FedLMFmm}.

\begin{algorithm}[!htpb]
\caption{FedMC-ADMM}
\label{FedLMFmm}
\SetKwInput{KwInput}{Input}
\SetKwRepeat{Do}{Repeat}{until}
\DontPrintSemicolon

\KwInput{$V^0 \in \mathbb{R}^{r \times n}$, $U_i^0 \in \mathbb{R}^{m_i \times r}$, $W_i^0 = V^0$, $Y_i^0 = -\frac{1}{p} \nabla_{V} f_i(U_i^0, W_i^0)$, $\beta > 0$, and $k = 0$.}

\Do{stopping criterion is met}{
\textbf{[Active users]} The server generates a subset of clients \(S_k \subset [p]\). 

\textbf{[Communication]} The server sends \(V^k\) to each client \(i \in S_k\). 

\textbf{[Local update]} For each client \(i \in S_k\): 
\begin{enumerate}[label={\rm(\alph*)}]
    \item Update \(U_i^{k+1}\) as follows
    \begin{enumerate}
        \item[(i)] For \(l = 1, \dots, N\)
        \begin{equation}\label{u:update}
            U_i^{k,l} \in \argmin_{U_i \in \mathcal{U}_i} \langle \nabla_U f_i(U_i^{k,l-1}, W_i^k), U_i \rangle + \frac{L_{W_i^k}}{2} \|U_i - U_i^{k,l-1}\|^2 + R_i(U_i),
        \end{equation}
        where \(U_i^{k,0} = U_i^k\).
        \item[(ii)] Set \(U_i^{k+1} = U_i^{k,N}\).
    \end{enumerate}

    \item Update \(W_i^{k+1}\) as follows
    \begin{enumerate}
        \item[(i)] For \(l = 1, \dots, N\):
        \begin{equation}\label{v:update}
            W_i^{k,l} = \frac{L_{U_i^{k+1}} / p \cdot W_i^{k,l-1} + \beta V^k - \nabla_V f_i(U_i^{k+1}, W_i^{k,l-1})/p - Y_i^k}{L_{U_i^{k+1}} / p + \beta},
        \end{equation}
        where \(W_i^{k,0} = W_i^k\).
        \item[(ii)] Set \(W_i^{k+1} = W_i^{k,N}\).
    \end{enumerate}

    \item Set \(Y_i^{k+1} = Y_i^k + \beta (W_i^{k+1} - V^k)\).
\end{enumerate}
\textbf{[Communication]} Each client \(i \in S_k\) sends \((W_i^{k+1}, Y_i^{k+1})\) back to the server. 

\textbf{[Server update]} The server updates \(V^{k+1}\) as a solution to the problem
\begin{equation}\label{s:update}
    \min_{V \in \mathcal{V}}\quad  \sum_{i=1}^p \left[\langle Y_i^{k+1}, W_i^{k+1} - V \rangle + \frac{\beta}{2} \|W_i^{k+1} - V\|^2 \right] + R(V)
\end{equation}
with $(W_i^{k+1},Y_i^{k+1})=(W_i^{k},Y_i^{k})$ for $i\notin S_k$.

\textbf{Set \(k \gets k + 1\)}. 
}
\end{algorithm}

\subsection{Convergence Analysis}

Note that Problem \eqref{model} is equivalent to the following unconstrained optimization problem
\begin{equation}\label{unconstrained}
\min_{\substack{U_i \in \mathbb{R}^{m_i \times r}, V \in \mathbb{R}^{r \times n}}} \Phi(\mbfU, V) := \frac{1}{p} \sum_{i=1}^p \big[F_i(M_i, U_iV) + \mathcal{I}_{\mathcal{U}_i}(U_i)\big] + R(V) + \mathcal{I}_{\mathcal{V}}(V),
\end{equation}
where \(\mathcal{I}_{\mathcal{U}_i}(\cdot)\),  \(i \in [p]\) is the indicator function of \(\mathcal{U}_i\). As the federated matrix completion problem in \eqref{model} is usually nonconvex, it is meaningful to find  the critical/stationary solutions \((\mbfU^*, V^*)\) satisfying the optimality condition \(0 \in \partial \Phi(\mbfU^*, V^*)\) instead of the global minimizer. In this section, we  focus on the convergence analysis and complexity of Algorithm~\ref{FedLMFmm}. To establish theoretical guarantees, the following standard assumptions are adopted for Problem~\eqref{model}.

\begin{assumption}\label{basic:assumption}
The following conditions hold
\begin{enumerate}[label={\rm(\alph*)}] 
    \item The functions \( R_i \) and \( R \) are l.s.c. convex. The sets \(\operatorname{dom}\, R_i \cap \mathcal{U}_i\) and \(\operatorname{dom}\, R \cap \mathcal{V}\) are non-empty. Moreover,  \(\upsilon = \inf\limits_{\mbfU \in \mathcal{U}, V \in \mathcal{V}} \sum_{i=1}^p F_i(U_i,V) > -\infty\), and \(\alpha = \inf\limits_{\mathbf{U} \in \mathcal{U}, W_i\in\mathbb R^{r\times n}, V \in \mathcal{V}, Y\in \mathbb R^{r\times n}} \mathcal{L}(\mathbf{U}, \mbfW, V, \mathbf{Y}) > -\infty\).
    
    \item For any fixed $V$, the function $f_i(\cdot, V)$ is $L_{V}$-smooth on $\mathcal{U}_i$, i.e., it is continuously differentiable, and there exists a positive constant $L_{V}$ such that
    \begin{equation}\label{average-smoothness}
        \|\nabla_U f_i(U, V) - \nabla_U f_i(X, V)\| \leq L_{V} \|U - X\|, \quad \forall U, X \in \mathcal{U}_i.
    \end{equation}
    Likewise, for any fixed $U_i$, the function $f_i(U_i, \cdot)$ is $L_{U_i}$-smooth on $\mathcal{V}$.

    \item There exists a positive constant \(L\) such that, for any fixed \(V\), the mapping \(\nabla_V f_i(\cdot, V)\) is \(L\)-Lipschitz on \(\mathcal{U}_i\) in the sense that 
    \[
    \|\nabla_V f_i(U, V) - \nabla_V f_i(X, V)\| \leq L \|U - X\|\quad \forall \; U,X\in \mathcal{U}_i, V\in \mathcal{V}.
    \] 
\end{enumerate}
\end{assumption}

To have a uniform sampling scheme of $S_k$ at every communication round, we also need to assume that each client $i>0$ is sampled with probability at least $p_i>0$. This assumption is broad and accommodates various sampling schemes, including non-overlapping uniform sampling and doubly uniform sampling, as discussed in \cite{richtarik2016parallel}.

\begin{assumption}\label{assump:sampling}
There exist probabilities \(p_1, p_2, \ldots, p_p > 0\) such that, at every communication round and independently of the past, each client \(i \in [p]\) is sampled with probability \(p_i\).
\end{assumption}

The following lemma is useful in our analysis.

\begin{lemma}\label{conditionalexpectation}
    Let \(\Theta_i^k\) for \(i \in [p]\) be random variables that are
    measurable with respect to the history
    up to round \(k\). Suppose Assumption~\ref{assump:sampling} holds.
    Then, for any \(k\), we have
    \[
        \E_k\!\biggl[\sum_{i \in S_k} \Theta_i^k\biggr]
        \;=\;
        \sum_{i=1}^p p_i\,\Theta_i^k,
    \]
    where \(\mathbb{E}_k[\cdot]\) denotes the conditional expectation
    given the history up to the sampling at round \(k\)
\end{lemma}
\begin{proof}
Since \(\Theta_i^k\) is \(\mathcal{F}_k\)-measurable for each \(i\in[p]\),
we can view \(\sum_{i \in S_k} \Theta_i^k\) as a random quantity where the only remaining randomness is due to the sampling of \(S_k\).
Hence, we have
\[
\begin{aligned}
    \E_k\biggl[\sum_{i \in S_k} \Theta_i^k\biggr]
    &= \sum_{S \subseteq [p]} \mathbb P\bigl(S_k = S \mid \mathcal{F}_k\bigr)
       \sum_{i \in S} \Theta_i^k \\
    &= \sum_{i=1}^p \sum_{S \ni i}
       \mathbb P\bigl(S_k = S \mid \mathcal{F}_k\bigr)\,\Theta_i^k.
\end{aligned}
\]
By Assumption~\ref{assump:sampling}, the probability that \(i\) is included
in \(S_k\) (independently of \(\mathcal{F}_k\)) is \(p_i\), namely
\[
\sum_{S \ni i} \mathbb P\bigl(S_k = S \mid \mathcal{F}_k\bigr) \;=\; p_i.
\]
Substituting back into the previous expression yields
\[
\E_k\biggl[\sum_{i \in S_k} \Theta_i^k\biggr]
\;=\;
\sum_{i=1}^p p_i\,\Theta_i^k,
\]
which completes the proof.
\end{proof}

Before proceeding to the convergence analysis, we initialize \(W_i^{-1, l} = W_i^0\) for any \(i \in [p]\) and \(l \in [N]\). For any client \(i \notin S_k\), we set \(W_i^{k, l} := W_i^{k-1, l}\) and \(U_i^{k, l} = U_i^{k, 0} := U_i^k\) for all \(l \in [N]\) and \(k \geq 0\). To facilitate the analysis, let us introduce \(\bar{U}_i^{k, l}\), a virtual variable defined as follows
\begin{equation}\label{1122g}
    \bar{U}_i^{k, l} = \argmin_{U_i \in \mathcal{U}_i}\;\; \langle \nabla_U f_i(U_i^{k, l-1}, W_i^k), U_i \rangle + \frac{L_{W_i^k}}{2} \|U_i - U_i^{k, l-1}\|^2 + R_i(U_i).
\end{equation}
This is exactly the same update \eqref{125a} when $i\in S_k$, i.e., $\bar{U}_i^{k, l}={U}_i^{k, l}$ for $l\in[N]$ in this case. We do not use $\bar{U}_i^{k, l}$ when $i\notin S_k$ in the $k$-th round of Algorithm~\ref{FedLMFmm}, but employ them for convergence analysis.

The following useful lemma establishes the relationship between $U_i^k$, $W_i^k$, and $Y_i^k$.

\begin{lemma}\label{lemmay}
Let \(\{U_i^k, W_i^k, Y_i^k\}\) be generated by Algorithm~\ref{FedLMFmm}. Then, for all \(i \in [p]\) and \(k \geq 0\), we have
\begin{equation}\label{lemya}
    Y_i^k = - \frac{1}{p}\left[\nabla_V f_i(U_i^k, W_i^{k-1,N-1}) + L_{U_i^k}(W_i^k - W_i^{k-1,N-1})\right].
\end{equation}
\end{lemma}

\begin{proof}
We prove \eqref{lemya} by induction. For the base case \(k = 0\), \eqref{lemya} follows directly from the initialization of Algorithm~\ref{FedLMFmm} with \(Y_i^0 = -\frac{1}{p}\nabla_V f_i(U_i^0, W_i^0)\) and \(W_i^{-1,l} = W_i^0\) for all \(i \in [p]\) and \(l \in [N]\).

Let us assume that \eqref{lemya} holds for some \(k \geq 0\). We need to show that \eqref{lemya} is also satisfied for \(k+1\) in the following two cases:

\noindent {\bf Case 1:} \(i \in S_k\). The update rule for \(W_i^{k,l}\) in \eqref{v:update} implies
\begin{equation*}
    \frac{1}{p}\nabla_V f_i(U_i^{k+1}, W_i^{k,l-1}) + Y_i^k + \frac{1}{p}L_{U_i^{k+1}}(W_i^{k,l} - W_i^{k,l-1}) + \beta(W_i^{k,l} - V^k) = 0.
\end{equation*}
By choosing \(l = N\), this together with  the definition of \(Y_i^{k+1}\) gives us that
\begin{equation*}
\begin{aligned}
    Y_i^{k+1} &= Y_i^k+\beta(W_i^{k,N}-V^k)\\
   &= - \frac{1}{p}\left[\nabla_V f_i(U_i^{k+1}, W_i^{k,N-1}) + L_{U_i^{k+1}}(W_i^{k+1} - W_i^{k,N-1})\right].
\end{aligned}
\end{equation*}

\noindent{\bf Case 2:} \(i \notin S_k\). We have \(Y_i^{k+1} = Y_i^k\), \((U_i^{k+1},W_i^{k+1}) = (U_i^k,W_i^k)\), and \(W_i^{k,N-1} = W_i^{k-1,N-1}\). The inductive hypothesis \eqref{lemya} yields that 
\begin{align*}
    Y_i^{k+1} &= - \frac{1}{p}\left[\nabla_V f_i(U_i^k, W_i^{k-1,N-1}) + L_{U_i^k}(W_i^k - W_i^{k-1,N-1})\right] \\
    & = - \frac{1}{p}\left[\nabla_V f_i(U_i^{k+1}, W_i^{k,N-1}) + L_{U_i^{k+1}}(W_i^{k+1} - W_i^{k,N-1})\right]. 
\end{align*}
In both cases, \eqref{lemya} holds for \(k+1\). This completes the proof.
\end{proof}

The next lemma plays an important role in our analysis. 

\begin{lemma}\label{lemmaL}
Let $\{\mathbf{U}^k, \mbfW^k, V^k, \mathbf{Y}^k\}$ be a sequence generated by Algorithm~\ref{FedLMFmm}. Then, for any \(k\), we have
\begin{equation}\label{ineL}
    \begin{split}
        & \mathcal{L}^{k+1} + \frac{p\beta}{2} \|V^{k+1} - V^k\|^2 + \sum_{i=1}^p\sum_{l=1}^N \biggl[\biggl(\frac{L_{W_i^k}}{2p}-\frac{4NL^2}{p^2\beta}\biggr)\|U_i^{k,l} - U_i^{k,l-1}\|^2\\
        & \quad + \biggl(\frac{L_{U_i^{k+1}}}{2p} - \frac{16L_{U_i^{k+1}}^2+4NL_{U_i^k}^2}{p^2\beta}\biggr)\|W_i^{k,l} - W_i^{k,l-1}\|^2\biggr] \\
        & \leq \mathcal{L}^k + \sum_{i=1}^p \sum_{l=1}^N\frac{16L_{U_i^k}^2}{p^2\beta}\|W_i^{k-1,l} - W_i^{k-1,l-1}\|^2 + \sum_{i \notin S_k}\sum_{l=1}^N \frac{L_{U_i^k} }{2p}\|W_i^{k-1,l} - W_i^{k-1,l-1}\|^2,
    \end{split}
\end{equation}
where \(\mathcal{L}^k := \mathcal{L}(\mathbf{U}^k, \mbfW^k, V^k, \mathbf{Y}^k)\).
\end{lemma}

\begin{proof}
Let us fix \(k \in \mathbb{N}\). By using the update rule for \(U_i^{k,l}\) in \eqref{u:update}, for any \(i \in S_k\) and \(l \in [N]\), we obtain from Lemma~\ref{convexity-smoothness}(c) that
\begin{equation}\label{1216a}
    \langle \nabla_U f_i(U_i^{k,l-1}, W_i^k),  U_i^{k,l} \rangle +R_i(U_i^{k,l}) + L_{W_i^k} \|U_i^{k,l} - U_i^{k,l-1}\|^2 \leq \langle \nabla_U f_i(U_i^{k,l-1}, W_i^k), U_i^{k,l-1} \rangle + R_i(U_i^{k,l-1}).
\end{equation}
By Lemma~\ref{convexity-smoothness}(d), the \(L_{W_i^k}\)-smoothness of \(f_i(\cdot, W_i^k)\) gives us that
\begin{equation}\label{1216b}
   f_i(U_i^{k,l}, W_i^k) \leq  f_i(U_i^{k,l-1}, W_i^k) + \langle \nabla_U f_i(U_i^{k,l-1}, W_i^k),  U_i^{k,l}-U_i^{k,l-1}\rangle  + \frac{L_{W_i^k}}{2}\|U_i^{k,l} - U_i^{k,l-1}\|^2.
\end{equation}
By adding \eqref{1216a} to \eqref{1216b}, we arrive at
\begin{equation*}
  f_i(U_i^{k,l}, W_i^k)+R_i(U_i^{k,l}) + \frac{L_{W_i^k}}{2} \|U_i^{k,l} - U_i^{k,l-1}\|^2\leq   f_i(U_i^{k,l-1}, W_i^k) + R_i(U_i^{k,l-1})
\end{equation*}
Telescoping this inequality over \(l = 1, \dots, N\) leads us to
\begin{equation*}
  f_i(U_i^{k+1}, W_i^k)+R_i(U_i^{k+1}) + \sum_{l=1}^N\frac{L_{W_i^k}}{2} \|U_i^{k,l} - U_i^{k,l-1}\|^2\leq   f_i(U_i^k, W_i^k) + R_i(U_i^k).
\end{equation*}
By combining this with \eqref{def:AL}, we have
\begin{equation}\label{1216d}
    \mathcal{L}_i(U_i^{k+1}, W_i^k, V^k, Y_i^k) + \sum_{l=1}^N \frac{L_{W_i^k}}{2p} \|U_i^{k,l} - U_i^{k,l-1}\|^2 \leq \mathcal{L}_i(U_i^k, W_i^k, V^k, Y_i^k).
\end{equation}
Following the update rule for \(W_i^{k,l}\) in \eqref{v:update} gives us that
\[
Y_i^k+L_{U_i^{k+1}}/p(W_i^{k,l}-W_i^{k,l-1})+\beta(W_i^{k,l}-V^k)=-\nabla_Vf_i(U_i^{k+1}, W_i^{k,l-1})/p,
\]
which implies that 
\begin{equation}\label{eq:YWf}
\langle Y_i^k+L_{U_i^{k+1}}/p(W_i^{k,l}-W_i^{k,l-1})+\beta(W_i^{k,l}-V^k),W_i^{k,l}-W_i^{k,l-1}\rangle=\langle \nabla_Vf_i(U_i^{k+1}, W_i^{k,l-1})/p,W_i^{k,l-1}-W_i^{k,l}\rangle.
\end{equation}
Note further  that 
\[
\frac{\beta}{2}\|W_i^{k,l}-V^k\|^2\le \beta\langle W_i^{k,l}-V^k,W_i^{k,l}-W_i^{k,l-1}\rangle+\frac{\beta}{2}\|W_i^{k,l-1}-V^k\|^2.
\]
Combining this with \eqref{eq:YWf}  deduces that 
\begin{equation}\label{1216e}
    \begin{split}
        &\langle Y_i^k, W_i^{k,l} - V^k \rangle  + L_{U_i^{k+1}} / p \|W_i^{k,l} - W_i^{k,l-1}\|^2+ \frac{\beta}{2} \|W_i^{k,l} - V^k\|^2 \\
        \le &\langle Y_i^k, W_i^{k,l} - V^k \rangle  + L_{U_i^{k+1}} / p \|W_i^{k,l} - W_i^{k,l-1}\|^2+ \beta\langle W_i^{k,l}-V^k,W_i^{k,l}-W_i^{k,l-1}\rangle+\frac{\beta}{2}\|W_i^{k,l-1}-V^k\|^2\\
        =&\langle \nabla_V f_i(U_i^{k+1}, W_i^{k,l-1}) / p, W_i^{k,l-1} - W_i^{k,l} \rangle + \langle Y_i^k, W_i^{k,l-1} - V^k \rangle + \frac{\beta}{2} \|W_i^{k,l-1} - V^k\|^2.
    \end{split}
\end{equation}
Due to the \(L_{U_i^{k+1}}\)-smoothness of \(f_i(U_i^{k+1}, \cdot)\) and Lemma~\ref{convexity-smoothness}(d), we have
\[
\langle \nabla_V f_i(U_i^{k+1}, W_i^{k,l-1}), W_i^{k,l-1} - W_i^{k,l} \rangle \le f(U_i^{k+1},W_i^{k,l-1})-f(U_i^{k+1},W_i^{k,l})+\frac{L_{U_i^{k+1}}}{2}\|W_i^{k,l-1} - W_i^{k,l}\|^2.
\]
This together with \eqref{1216e} implies that
\begin{equation*}\label{1216f}
    \mathcal{L}_i(U_i^{k+1}, W_i^{k,l}, V^k, Y_i^k) + \frac{L_{U_i^{k+1}}}{2p} \|W_i^{k,l} - W_i^{k,l-1}\|^2 \leq \mathcal{L}_i(U_i^{k+1}, W_i^{k,l-1}, V^k, Y_i^k).
\end{equation*}
Telescoping this inequality over \(l = 1, \dots, N\) leads us to
\begin{equation}\label{1216g}
    \mathcal{L}_i(U_i^{k+1}, W_i^{k+1}, V^k, Y_i^k) + \sum_{l=1}^N \frac{L_{U_i^{k+1}}}{2p} \|W_i^{k,l} - W_i^{k,l-1}\|^2 \leq \mathcal{L}_i(U_i^{k+1}, W_i^k, V^k, Y_i^k).
\end{equation}
By summing up \eqref{1216d} and \eqref{1216g}, we obtain
\begin{equation*}
    \mathcal{L}_i(U_i^{k+1}, W_i^{k+1}, V^k, Y_i^k) + \sum_{l=1}^N \left[\frac{L_{W_i^k}}{2p}\|U_i^{k,l} - U_i^{k,l-1}\|^2+ \frac{L_{U_i^{k+1}} }{2p}\|W_i^{k,l} - W_i^{k,l-1}\|^2\right] \leq \mathcal{L}_i(U_i^k, W_i^k, V^k, Y_i^k).
\end{equation*}
Recall that for any \(i \notin S_k\), \(W_i^{k,l} = W_i^{k-1,l}\) for all \(l=0,1,\ldots, N\), \(W_i^{k+1} = W_i^k\), and \(U_i^{k,l} = U_i^{k,l-1}\) for all \(l \in [N]\). It follows from the above inequality and \eqref{def:AugL} that 
\begin{equation}\label{1216j}
    \begin{split}
        & \mathcal{L}(\mbfU^{k+1}, \mbfW^{k+1}, V^k, \mbfY^k) + \sum_{i=1}^p\sum_{l=1}^N \left[\frac{L_{W_i^k}}{2p}\|U_i^{k,l} - U_i^{k,l-1}\|^2+ \frac{L_{U_i^{k+1}}}{2p}\|W_i^{k,l} - W_i^{k,l-1}\|^2\right] \\
        & \leq \mathcal{L}^k + \sum_{i \notin S_k}\sum_{l=1}^N \frac{L_{U_i^k}}{2p}\|W_i^{k-1,l} - W_i^{k-1,l-1}\|^2.
    \end{split}
\end{equation}
On the other hand, from Lemma \ref{lemmay}, for any \(i \in [p]\), we have
\[
\begin{split}
     Y_i^{k+1} - Y_i^k
     & = \frac{1}{p} \big[\nabla_V f_i(U_i^k, W_i^{k-1,N-1}) - \nabla_V f_i(U_i^{k+1}, W_i^{k,N-1}) \\
     & \quad + L_{U_i^k}(W_i^k - W_i^{k-1,N-1}) - L_{U_i^{k+1}}(W_i^{k+1} - W_i^{k,N-1})\big]\\
     & \le \frac{1}{p} \big[\nabla_V f_i(U_i^k, W_i^{k-1,N-1}) - \nabla_V f_i(U_i^k, W_i^k) \\
     &\quad +  \nabla_V f_i(U_i^{k+1}, W_i^{k+1}) - \nabla_V f_i(U_i^{k+1}, W_i^{k,N-1})\\
     &\quad + \nabla_V f_i(U_i^k, W_i^k) - \nabla_V f_i(U_i^k, W_i^{k+1}) \\
     &\quad + \nabla_V f_i(U_i^k, W_i^{k+1}) - \nabla_V f_i(U_i^{k+1}, W_i^{k+1})\\
     & \quad + L_{U_i^k}\|W_i^k - W_i^{k-1,N-1}\|+  L_{U_i^{k+1}}\|W_i^{k+1} - W_i^{k,N-1}\|\big].
    \end{split}
\]
Due to the $L$-smoothness of $\nabla_V f_i$ in Assumption~\ref{basic:assumption}(c), the above inequality implies that 
\begin{equation}\label{1216k}
    \begin{split}
        p\|Y_i^{k+1} - Y_i^k\| \leq  &\ 2L_{U_i^k}\|W_i^k - W_i^{k-1,N-1}\| + 2L_{U_i^{k+1}}\|W_i^{k+1} - W_i^{k,N-1}\| + L_{U_i^k}\|W_i^{k+1} - W_i^k\| \\
        & + L\|U_i^{k+1} - U_i^k\|.
    \end{split}
\end{equation}
 By squaring both sides and applying the Cauchy–Schwarz inequality, we derive
\begin{equation}\label{1216l}
\begin{aligned}
    p^2\|Y_i^{k+1} - Y_i^k\|^2 \leq &\ 16L_{U_i^k}^2\|W_i^k - W_i^{k-1,N-1}\|^2 + 16L_{U_i^{k+1}}^2\|W_i^{k+1} - W_i^{k,N-1}\|^2 \\ &+ 4L_{U_i^k}^2\|W_i^{k+1} - W_i^k\|^2 + 4L^2\|U_i^{k+1} - U_i^k\|^2\\
        \leq &\ \sum_{l=1}^N\biggl[16L_{U_i^k}^2\|W_i^{k-1,l} - W_i^{k-1,l-1}\|^2 + (16L_{U_i^{k+1}}^2+4NL_{U_i^k}^2)\|W_i^{k,l} - W_i^{k,l-1}\|^2  \\
        & + 4NL^2\|U_i^{k,l} - U_i^{k,l-1}\|^2\biggr].
\end{aligned}
\end{equation}
From the definition of \(\mathcal{L}_i\) in \eqref{def:AL}, for any \(i \in [p]\), note that
\begin{equation*}
    \mathcal{L}_i(U_i^{k+1}, W_i^{k+1}, V^k, Y_i^{k+1}) = \mathcal{L}_i(U_i^{k+1}, W_i^{k+1}, V^k, Y_i^k) + \langle Y_i^{k+1} - Y_i^k, W_i^{k+1} - V^k \rangle.
\end{equation*}
Recall that \(Y_i^{k+1} - Y_i^k = \beta(W_i^{k+1} - V^k)\) if \(i\in S_k\) and \(Y_i^{k+1} - Y_i^k = 0\) if \(i\notin S_k\). For any \(i \in [p]\), it follows from that
\begin{equation}\label{1216n}
     \mathcal{L}(\mathbf{U}^{k+1}, \mbfW^{k+1}, V^k, \mathbf{Y}^{k+1}) = \mathcal{L}(\mathbf{U}^{k+1}, \mbfW^{k+1}, V^k, \mathbf{Y}^k) + \sum_{i=1}^p \frac{\|Y_i^{k+1} - Y_i^k\|^2}{\beta}.
\end{equation}
Combining \eqref{1216l} and \eqref{1216n} gives us that
\begin{equation}\label{1216o}
    \begin{split}
        \mathcal{L}(\mathbf{U}^{k+1}, \mbfW^{k+1}, V^k, \mathbf{Y}^{k+1}) & \leq \mathcal{L}(\mathbf{U}^{k+1}, \mbfW^{k+1}, V^k, \mathbf{Y}^k)  + \sum_{i=1}^p \sum_{l=1}^N\biggl[\frac{16L_{U_i^k}^2}{p^2\beta}\|W_i^{k-1,l} - W_i^{k-1,l-1}\|^2\\
        &  + \frac{16L_{U_i^{k+1}}^2+4NL_{U_i^k}^2}{p^2\beta}\|W_i^{k,l} - W_i^{k,l-1}\|^2  + \frac{4NL^2}{p^2\beta}\|U_i^{k,l} - U_i^{k,l-1}\|^2\biggr].
    \end{split}
\end{equation}
We obtain from \eqref{1216j} and \eqref{1216o} that
\begin{equation}\label{1216p}
    \begin{split}
        & \mathcal{L}(\mathbf{U}^{k+1}, \mbfW^{k+1}, V^k, \mathbf{Y}^{k+1}) + \sum_{i=1}^p\sum_{l=1}^N \biggl[\biggl(\frac{L_{W_i^k}}{2p}-\frac{4NL^2}{p^2\beta}\biggr)\|U_i^{k,l} - U_i^{k,l-1}\|^2\\
        & \quad + \biggl(\frac{L_{U_i^{k+1}}}{2p} - \frac{16L_{U_i^{k+1}}^2+4NL_{U_i^k}^2}{p^2\beta}\biggr)\|W_i^{k,l} - W_i^{k,l-1}\|^2\biggr] \\
        & \leq \mathcal{L}^k + \sum_{i=1}^p \sum_{l=1}^N\frac{16L_{U_i^k}^2}{p^2\beta}\|W_i^{k-1,l} - W_i^{k-1,l-1}\|^2 + \sum_{i \notin S_k}\sum_{l=1}^N \frac{L_{U_i^k} }{2p}\|W_i^{k-1,l} - W_i^{k-1,l-1}\|^2.
    \end{split}
\end{equation}
Moreover, by the update rule for \(V^{k+1}\) in \eqref{s:update} and  Lemma~\ref{convexity-smoothness}(c) again, we have
\begin{equation*}
    \mathcal{L}^{k+1} + \frac{p\beta}{2} \|V^{k+1} - V^k\|^2=\mathcal{L}(\mathbf{U}^{k+1}, \mbfW^{k+1}, V^{k+1}, \mathbf{Y}^{k+1}) + \frac{p\beta}{2} \|V^{k+1} - V^k\|^2  \leq \mathcal{L}(\mathbf{U}^{k+1}, \mbfW^{k+1}, V^k, \mathbf{Y}^{k+1}).
\end{equation*}
Adding this inequality with \eqref{1216p} verifies the inequality \eqref{ineL}. The proof is complete.
\end{proof}

We are ready to set up our main convergence results. The following theorem  guarantees almost sure subsequential convergence of the sequence generated by our Algorithm~\ref{FedLMFmm} to a stationary point.

\begin{theorem}[Almost Sure Subsequential Convergence]\label{t:ASC}
Let $\{\mathbf{U}^k, \mbfW^k, V^k, \mathbf{Y}^k\}$ be a sequence generated by Algorithm~\ref{FedLMFmm}. Suppose that Assumptions~\ref{basic:assumption} and~\ref{assump:sampling} are satisfied, and that $L_{U_i^k}$ and $L_{W_i^k}$ are bounded, i.e., 
\begin{equation}\label{eq:ULB}
\underaccent{\bar}{L}_U \leq L_{U_i^k} \leq \bar{L}_U \quad \text{and} \quad \underaccent{\bar}{L}_V \leq L_{W_i^k} \leq \bar{L}_V,
\end{equation}
for some constants $\underaccent{\bar}{L}_U, \bar{L}_U, \underaccent{\bar}{L}_V, \bar{L}_V > 0$. Additionally, suppose that $p_{\min}: = \min p_i>0$ and 
\[
\beta > \min\left\{\frac{8NL^2}{p\underaccent{\bar} L_V}, \frac{8(8+N)\bar L_U^2}{p_{\min}p\underaccent{\bar} L_U}\right\}.
\]
Then
\begin{enumerate}[label={\rm(\alph*)}]
    \item\label{t:ASC-b} The sequences $\{\sum_{i=1}^p\sum_{l=1}^N\|U_i^{k,l}-U_i^{k,l-1}\|^2\}$, $\{\sum_{i=1}^p\sum_{l=1}^N\|W_i^{k,l}-W_i^{k,l-1}\|^2\}\}$, $\{\|V^k - V^{k-1}\|^2\}$, and \(\{\sum_{i=1}^p \|W_i^k - V^k\|^2\}\) have finite sums (and, in particular, vanish) almost surely.
    \item\label{t:ASC-c} If $(\mathbf{U}^*, \mbfW^*, V^*, \mathbf{Y}^*)$ is a limit point of $\{\mathbf{U}^k, \mbfW^k, V^k, \mathbf{Y}^k\}$, then $(\mbfU^*, V^*)$ is a stationary point of $\Phi$, almost surely.
\end{enumerate}
\end{theorem}

\begin{proof}
Let us start to justify \ref{t:ASC-b}. Denote $\hat{\mathcal L}^{k+1}$ by
\begin{equation*}
\begin{split}
     \hat{\mathcal L}^{k+1} := &\ \mathcal L^{k+1}  + \frac{p\beta}{2}\|V^{k+1}-V^k\|^2 + \sum_{i=1}^p\sum_{l=1}^N \biggl[\biggl(\frac{L_{W_i^k}}{2p}-\frac{4NL^2}{p^2\beta}\biggr)\|U_i^{k,l} - U_i^{k,l-1}\|^2\\
        &  + \biggl(\frac{L_{U_i^{k+1}}}{2p} - \frac{16
        L_{U_i^{k+1}}^2+4NL_{U_i^k}^2}{p^2\beta}\biggr)\|W_i^{k,l} - W_i^{k,l-1}\|^2\biggr].
\end{split}
\end{equation*}
By Lemma \ref{lemmaL}, we have
\begin{equation*}\label{1122c}
    \begin{split}
        \hat{\mathcal{L}}^{k+1} & \leq \hat{\mathcal{L}}^k - \frac{p\beta}{2} \|V^k - V^{k-1}\|^2 - \sum_{i=1}^p\sum_{l=1}^N\biggl(\frac{L_{W_i^{k-1}}}{2p}-\frac{4NL^2}{p^2\beta}\biggr)\|U_i^{k-1,l} - U_i^{k-1,l-1}\|^2 \\
        & + \sum_{i=1}^p\sum_{l=1}^N \frac{32L_{U_i^k}^2+4NL_{U_i^{k-1}}^2}{p^2\beta}\|W_i^{k-1,l} - W_i^{k-1,l-1}\|^2  - \sum_{i \in S_k}\sum_{l=1}^N \frac{L_{U_i^k}}{2p}\|W_i^{k-1,l} - W_i^{k-1,l-1}\|^2 .
    \end{split}
\end{equation*}
Taking the conditional expectation $\E_k$ (given the history up to iteration $k$) and applying Lemma \ref{conditionalexpectation} leads us to
\begin{equation}\label{1126a}
    \begin{split}
        \E_k\hat{\mathcal{L}}^{k+1} & \leq \hat{\mathcal{L}}^k - \frac{p\beta}{2} \|V^k - V^{k-1}\|^2 - \sum_{i=1}^p\sum_{l=1}^N\biggl(\frac{L_{W_i^{k-1}}}{2p}-\frac{4NL^2}{p^2\beta}\biggr)\|U_i^{k-1,l} - U_i^{k-1,l-1}\|^2 \\
        &  - \sum_{i=1}^p\sum_{l=1}^N \biggl(\frac{p_iL_{U_i^k}}{2p} - \frac{32L_{U_i^k}^2+4NL_{U_i^{k-1}}^2}{p^2\beta}\biggr)\|W_i^{k-1,l} - W_i^{k-1,l-1}\|^2.
    \end{split}
\end{equation}
Since $0< p_{\rm min}\leq p_i<1,$ $\underaccent{\bar}{L}_U \leq L_{U_i^k} \leq \bar{L}_U$, $\underaccent{\bar}{L}_V \leq L_{W_i^k} \leq \bar{L}_V$,  and \(\beta > \min\left\{\frac{8NL^2}{p{\underaccent{\bar}{L}}_V}, \frac{8(8+N)\bar L_U^2}{p_{\min}p\underaccent{\bar}{L}_U}\right\}\), we have
\[
\frac{L_{W_i^{k-1}}}{2p}-\frac{4NL^2}{p^2\beta} \geq \frac{\underaccent{\bar} L_V}{2p}-\frac{4NL^2}{p^2\beta} > 0,
\]
and 
\begin{align*}
  \frac{p_iL_{U_i^k}}{2p} - \frac{32L_{U_i^k}^2+4NL_{U_i^{k-1}}^2}{p^2\beta} \geq \frac{p_{\min}\underaccent{\bar} L_U}{2p} - \frac{4(8+N)\bar L_U^2}{p^2\beta} > 0.
\end{align*}
Combining  this with \eqref{1126a} gives us
\begin{equation}\label{1216r}
\begin{split}
        \E_k\hat{\mathcal{L}}^{k+1} & \leq \hat{\mathcal{L}}^k - \frac{p\beta}{2} \|V^k - V^{k-1}\|^2 - \sum_{i=1}^p\sum_{l=1}^N\biggl(\frac{\underaccent{\bar} L_V}{2p}-\frac{4NL^2}{p^2\beta}\biggr)\|U_i^{k-1,l} - U_i^{k-1,l-1}\|^2 \\
        &  - \sum_{i=1}^p\sum_{l=1}^N \biggl(\frac{p_{\min}\underaccent{\bar} L_U}{2p} - \frac{4(8+N)\bar L_U^2}{p^2\beta}\biggr)\|W_i^{k-1,l} - W_i^{k-1,l-1}\|^2.
    \end{split} 
\end{equation}
Note that \(\hat{\mathcal{L}}^k \geq \alpha\) for all \(k\), where \(\alpha\) is the lower bound of \(\mathcal L\) from Assumption~\ref{basic:assumption}(a). As we can add any positive constant to problem ~\eqref{model} without changing its nature, it is possible to assume that \(\alpha \geq 0\). Consequently, from \eqref{1216r}, applying the super-martingale convergence theorem (Lemma~\ref{supermartingale}), we conclude that the sequences $\{\|V^k - V^{k-1}\|^2\}$, $\{\sum_{i=1}^p\sum_{l=1}^N\|U_i^{k-1,l}-U_i^{k-1,l-1}\|^2\}$, and $\{\sum_{i=1}^p\sum_{l=1}^N\|W_i^{k-1,l}-W_i^{k-1,l-1}\|^2\}\}$ almost surely have finite sums (in particular, they vanish), and $\{\hat{\mathcal L}^{k}\}$ almost surely converges to a nonnegative random variable $\hat{\mathcal L}^{\infty}$. 

Next, we prove that \(\{\sum_{i=1}^p \|W_i^k - V^k\|^2\}\) has a finite sum almost surely. Particularly, \( \|W_i^k - V^k\|\) converges to $0$ almost surely. From the update rule for \(Y_i^k\) and \eqref{1216l}, for any \(i \in S_k\), we have
\begin{equation}
    \begin{split}
        p^2\beta^2  \|W_i^k - V^k\|^2 & \leq 2 p^2\beta^2  \|W_i^{k+1} - W_i^k\|^2 + 2  p^2\beta^2 \|W_i^{k+1} - V^k\|^2 \\
        & = 2  p^2\beta^2 \|W_i^{k+1} - W_i^k\|^2 + 2 p^2 \|Y_i^{k+1} - Y_i^k\|^2 \\
        & \leq 2 p^2\beta^2  \|W_i^{k+1} - W_i^k\|^2 + 2 \sum_{l=1}^N\biggl[16\bar L_U^2\|W_i^{k-1,l} - W_i^{k-1,l-1}\|^2 \\
        & + 4(4+N)\bar L_U^2\|W_i^{k,l} - W_i^{k,l-1}\|^2  + 4NL^2\|U_i^{k,l} - U_i^{k,l-1}\|^2\biggr].
    \end{split}
\end{equation}
Consequently, this implies that 
\begin{equation*}
    \begin{split}
        \sum_{i \in S_k} p^2\beta^2  \|W_i^k - V^k\|^2 & \leq \sum_{i=1}^p 2 p^2\beta^2  \|W_i^{k+1} - W_i^k\|^2 + 2\sum_{i=1}^p\sum_{l=1}^N\biggl[16\bar L_U^2\|W_i^{k-1,l} - W_i^{k-1,l-1}\|^2 \\
        & + 4(4+N)\bar L_U^2\|W_i^{k,l} - W_i^{k,l-1}\|^2  + 4NL^2\|U_i^{k,l} - U_i^{k,l-1}\|^2\biggr].
    \end{split}
\end{equation*}
By taking the expectation conditioned on the history up to round $k$ and applying Lemma \ref{conditionalexpectation}, we arrive at
\begin{equation*}\label{1110d}
    \begin{split}
        \sum_{i=1}^p p_i p^2\beta^2  \|W_i^k - V^k\|^2 \leq &\ \E_k\biggl[\sum_{i=1}^p 2 p^2\beta^2  \|W_i^{k+1} - W_i^k\|^2 + 2\sum_{i=1}^p\sum_{l=1}^N\biggl[16\bar L_U^2\|W_i^{k-1,l} - W_i^{k-1,l-1}\|^2 \\
        & + 4(4+N)\bar L_U^2\|W_i^{k,l} - W_i^{k,l-1}\|^2  + 4NL^2\|U_i^{k,l} - U_i^{k,l-1}\|^2\biggr]\biggr].
    \end{split}
\end{equation*}
Taking the total expectation gives us that 
\begin{equation}\label{1126b}
    \begin{split}
        \sum_{i=1}^p p_i p^2\beta^2  \E\|W_i^k - V^k\|^2 & \leq \sum_{i=1}^p 2 p^2\beta^2  \E\|W_i^{k+1} - W_i^k\|^2 + \sum_{i=1}^p\sum_{l=1}^N\biggl[32\bar L_U^2\E\|W_i^{k-1,l} - W_i^{k-1,l-1}\|^2 \\
        & + 8(4+N)\bar L_U^2\E\|W_i^{k,l} - W_i^{k,l-1}\|^2  + 8NL^2\E\|U_i^{k,l} - U_i^{k,l-1}\|^2\biggr].
    \end{split}
\end{equation}
Note from the Cauchy-Schwarz inequality that 
\[
\|W_i^{k+1} - W_i^k\|^2\le \sum_{l=1}^N N\|W_i^{k,l} - W_i^{k,l-1}\|^2.
\]
By telescoping \eqref{1126b} over \(k = 0, \ldots, K\) together with the above inequality, we derive that
\begin{equation}\label{1126c}
    \begin{split}
        \sum_{k=0}^K \sum_{i=1}^p p_i p^2\beta^2  \E\|W_i^k - V^k\|^2 \leq &\ \sum_{k=0}^K \sum_{i=1}^p\sum_{l=1}^N\biggl[\left(2Np^2\beta^2 + 8(8+N)\bar L_U^2\right)\E\|W_i^{k,l} - W_i^{k,l-1}\|^2\\
        & + 8NL^2\E\|U_i^{k,l} - U_i^{k,l-1}\|^2\biggr].
    \end{split}
\end{equation}
Taking the total expectation in \eqref{1216r} and telescoping over \(k = 0, \ldots, K\) lead us to
\begin{equation}\label{1128a}
    \begin{split}
        & \sum_{k=0}^K \frac{p\beta}{2} \E\|V^k - V^{k-1}\|^2 + \sum_{k=0}^K \sum_{i=1}^p\sum_{l=1}^N\biggl(\frac{\underaccent{\bar} L_V}{2p}-\frac{4NL^2}{p^2\beta}\biggr)\E\|U_i^{k-1,l} - U_i^{k-1,l-1}\|^2 \\
        &  + \sum_{k=0}^K\sum_{i=1}^p\sum_{l=1}^N \biggl(\frac{p_{\min}\underaccent{\bar} L_U}{2p} - \frac{4(8+N)\bar L_U^2}{p^2\beta}\biggr)\E\|W_i^{k-1,l} - W_i^{k-1,l-1}\|^2 \\
        & \leq \E\hat{\mathcal{L}}^0 - \E\hat{\mathcal{L}}^{K+1} \\
        & \leq \mathcal{L}^0 - \alpha.
    \end{split}
\end{equation}
It follows from \eqref{1128a} and \eqref{1126c} that
\begin{equation}
    \sum_{k=0}^K \sum_{i=1}^p p_i p^2\beta^2  \E\|W_i^k - V^k\|^2 < \infty.
\end{equation}
As a result, \(\{\sum_{i=1}^p \|W_i^k - V^k\|^2\}\) has a finite sum, which tells us that  \(\|W_i^k - V^k\|\) converges to $0$ almost surely for any $i\in[p]$. We conclude part (a) of the theorem. 

It remains to prove \ref{t:ASC-c}.  First, we claim that 
\begin{equation}\label{1122f}
    \lim_{k\to+\infty}[\bar U_i^{k,1} - U_i^k] = 0 \quad\text{almost surely},
\end{equation}
where $\bar U_i^{k,1}$ is defined in \eqref{1122g}. To justify this claim, recall from \eqref{u:update} that \(\bar U_i^{k,1}=U_i^{k,1}\) for any \(i\in S_k\). It follows that
\begin{equation*}\label{1216s}
    \sum_{i\in S_k}\|\bar U_i^{k,1} - U_i^k\|^2 \leq \sum_{i=1}^p\sum_{l=1}^N\|U_i^{k,l} - U_i^{k,l-1}\|^2.
\end{equation*}
By taking the total expectation and applying Lemma \ref{conditionalexpectation} 
with $\Theta_i^k = \|\bar{U}_i^{k,1} - U_i^k\|^2$, we arrive at
\begin{equation}\label{1216t}
    \sum_{i=1}^pp_i\E\|\bar U_i^{k,1} - U_i^k\|^2 \leq \sum_{i=1}^p\sum_{l=1}^N\E\|U_i^{k,l} - U_i^{k,l-1}\|^2.
\end{equation}
The claim now follows from \eqref{1128a} and \eqref{1216t}.

From part \ref{t:ASC-b} and \eqref{1122f}, for any sequence \(\{(\mbfU^k, \mbfW^k, V^k, \mbfY^k)\}\) generated by Algorithm~\ref{FedLMFmm}, and for all \(i \in [p]\), we have
\begin{equation}\label{1114b}
   \begin{split}
       &\lim_{k \to +\infty} [V^{k+1} - V^k] = 0, \quad \lim_{k \to +\infty} [W_i^k - V^k] = 0, \quad \lim_{k \to +\infty} [\bar{U}_i^{k,1} - U_i^k] = 0, \\
       &\lim_{k \to +\infty} [W_i^{k,l} - W_i^{k,l-1}] = 0, \quad \text{and} \quad \lim_{k \to +\infty} [U_i^{k,l} - U_i^{k,l-1}] = 0,
   \end{split}
\end{equation}
almost surely. Let \((\mbfU^*, \mbfW^*, V^*, \mbfY^*)\) be a limit point of \(\{(\mbfU^k, \mbfW^k, V^k, \mbfY^k)\}\) in the sense that  there exists a subsequence \(\{(\mbfU^{k_j}, \mbfW^{k_j}, V^{k_j}, \mbfY^{k_j})\}\) such that
\[
(\mbfU^{k_j}, \mbfW^{k_j}, V^{k_j}, \mbfY^{k_j}) \to (\mbfU^*, \mbfW^*, V^*, \mbfY^*) \quad \text{as } j \to +\infty.
\]
From \eqref{1114b}, it follows that
\[
\lim_{j \to +\infty} V^{k_j+1} = \lim_{j \to +\infty} V^{k_j} = \lim_{j \to +\infty} W_i^{k_j+1} = \lim_{j \to +\infty} W_i^{k_j,l} = \lim_{j \to +\infty} W_i^{k_j} = V^*,
\]
and
\[
\lim_{j \to +\infty} \bar{U}_i^{k_j,1} = \lim_{j \to +\infty} U_i^{k_j} = U_i^*.
\]
This tells us that $\mbfW^*=[V^*;\ldots; V^*]$.  
From the update rule of \(Y_i^{k_j+1}\), for any \(i \in [p]\), we have
\[
\|Y_i^{k_j+1} - Y_i^{k_j}\| = \beta \|W_i^{k_j+1} - V_i^{k_j}\|,
\]
which implies that \(\{Y_i^{k_j+1}\}\) also converges to \(Y_i^*\). From the definition of \(\bar{U}_i^{k_j,1}\) in \eqref{1122g}, we have
\begin{equation}\label{1122h}
   R_i(\bar{U}_i^{k_j,1}) \leq R_i(U_i) + \langle \nabla_U f_i(U_i^{k_j}, W_i^{k_j}), U_i - \bar{U}_i^{k_j,1} \rangle + \frac{L_{W_i^{k_j}}}{2} \|U_i - U_i^{k_j}\|^2 - \frac{L_{W_i^{k_j}}}{2} \|\bar{U}_i^{k_j,1} - U_i^{k_j}\|^2
\end{equation}
for any $U_i\in \mathcal{U}_i$. 
By choosing \(U_i = U_i^*\) and letting \(j \to +\infty\), we obtain
\[
\limsup_{j \to +\infty} R_i(\bar{U}_i^{k_j,1}) \leq R_i(U_i^*)\quad \mbox{for any}\quad U_i\in \mathcal{U}_i.
\]
Since \(R_i\) is lower semi-continuous, \(R_i(\bar{U}_i^{k_j,1})\) converges to \(R_i(U_i^*)\). By letting \(j \to +\infty\) in \eqref{1122h}, we have
\[
R_i(U_i^*) \leq R_i(U_i) + \langle \nabla_U f_i(U_i^*, W_i^*), U_i - U_i^* \rangle + \frac{\bar L_V}{2} \|U_i - U_i^*\|^2,
\]
which means that  \(U_i^*\) solves the following optimization problem
\begin{equation}\label{1122j}
    \min_{U_i \in \mathcal{U}_i}\quad F_i(M_i,U_i V^*)+ \langle \nabla_U f_i(U_i^*, V^*), U_i - U_i^* \rangle + \frac{\bar L_V}{2} \|U_i - U_i^*\|^2 - f_i(U_i, V^*).
\end{equation}
By writing the optimality condition for \eqref{1122j}, we obtain \(0 \in \partial_{U_i} \Phi(\mbfU^*, V^*)\) with $\Phi$ being defined in \eqref{unconstrained}, as $f_i$ is continuously differentiable.

Next, let us choose \(k = k_j + 1\) in \eqref{lemya} and let \(j \to +\infty\). We have 
\begin{equation}\label{1122m}
    Y_i^* = -\frac{1}{p} \nabla_V f_i(U_i^*, V^*).
\end{equation}
Note from the definition of \(V^{k_j+1}\) that
\begin{equation}\label{1122k}
    R(V^{k_j+1}) \leq R(V) + \sum_{i=1}^p \left[\langle Y_i^{k_j+1}, V^{k_j+1} - V \rangle + \frac{\beta}{2} \|W_i^{k_j+1} - V\|^2 - \frac{\beta}{2} \|W_i^{k_j+1} - V^{k_j+1}\|^2 \right]\quad \forall\, V \in \mathcal{V}.
\end{equation}
Plugging \(V = V^*\) into \eqref{1122k} and letting \(j \to +\infty\) give us that 
\[
\limsup_{j \to +\infty} R(V^{k_j+1}) \leq R(V^*).
\]
Since \(R\) is lower semi-continuous, the above inequality implies that \(R(V^{k_j+1})\) also converges to \(R(V^*)\). By taking \(j \to +\infty\) in \eqref{1122k}, we arrive at
\begin{equation}\label{1122l}
    R(V^*) \leq R(V) + \sum_{i=1}^p \left[\langle Y_i^*, V^* - V \rangle + \frac{\beta}{2} \|V^* - V\|^2 \right]\quad \mbox{for any}\quad V\in \mathcal{V}.
\end{equation}
Combining \eqref{1122m} with \eqref{1122l} leads us to
\[
R(V^*) \leq R(V) + \sum_{i=1}^p \left[\frac{1}{p}\langle  \nabla_V f_i(U_i^*, V^*), V - V^* \rangle + \frac{\beta}{2} \|V^* - V\|^2 \right].
\]
Consequently, \(V^*\) solves the following optimization problem
\begin{equation}
    \min_{V \in \mathcal{V}}\quad   F(\mbfU^*, V) + \sum_{i=1}^p \left[\frac{1}{p}\langle  \nabla_V f_i(U_i^*, V^*), V - V^* \rangle + \frac{\beta}{2} \|V^* - V\|^2 \right] - \sum_{i=1}^p \frac{1}{p} f_i(U_i^*, V).
\end{equation}
By writing its optimality at $V^*$, we also have  \(0 \in \partial_V \Phi(\mbfU^*, V^*)\). It follows that $(0,0)\in \partial \Phi(\mbfU^*, V^*)$, i.e., $(\mbfU^*, V^*)$ is a stationary/critical point of problem \(\Phi\) in \eqref{unconstrained}. The proof is complete. 
\end{proof}


Finally, we establish the communication round complexity required to achieve a stationary point.

\begin{theorem}[Round complexity]\label{thorem2}
Let $\{\mathbf{U}^k, \mbfW^k, V^k, \mathbf{Y}^k\}$ be a sequence generated by Algorithm~\ref{FedLMFmm}. Suppose that Assumptions~\ref{basic:assumption} and~\ref{assump:sampling} are satisfied, and that \(L_{U_i^k}\) and \(L_{W_i^k}\) are bounded, i.e., 
\[
\underaccent{\bar}{L}_U \leq L_{U_i^k} \leq \bar{L}_U \quad \text{and} \quad \underaccent{\bar}{L}_V \leq L_{W_i^k} \leq \bar{L}_V,
\]
for some constants \(\underaccent{\bar}{L}_U, \bar{L}_U, \underaccent{\bar}{L}_V, \bar{L}_V > 0\). Additionally, suppose that for any fixed \(U_i\), the mapping \(\nabla_U f_i(U_i, \cdot)\) is \(\kappa\)-Lipschitz on \(\mathcal{V}\) for some \(\kappa > 0\), $p_{\min} = \min p_i>0$ and 
\[
\beta > \min\left\{\frac{8NL^2}{p\underaccent{\bar} L_V}, \frac{8(8+N)\bar L_U^2}{p_{\min}p\underaccent{\bar} L_U}\right\}.
\]
Then, for any positive integer \(K\), we have 
\begin{equation}\label{100724a}
\frac{1}{K}\sum_{k=1}^K\E \dist\left(0,\partial \Phi(\bar{\mbfU}^{k,1},V^k)\right)^2
    \leq \frac{C(\mathcal L^0 - \alpha)}{K},
\end{equation}
where 
\[
\begin{split}
   C : =  &\ \biggl[\frac{4(N\beta^2p^2 + 4(8+N)\bar L_U^2)(3(\bar L_U^2 +p\beta)^2p + 2\kappa^2) + 24p_{\min}p^3\beta^2\bar L_U^2}{p_{\min}p\beta\underaccent{\bar} L_U - 8(8+N)\bar L_U^2}\\
        & + \frac{6(8N(\bar L_U^2 +p\beta)^2 + p^2\beta^2)L^2p + 32NL^2\kappa^2 + 16p^2\beta^2\bar L_V^2}{p\beta\underaccent{\bar} L_V - 8NL^2} \biggr]\frac{1}{p_{\min}p^2\beta}.
\end{split}
\]
Consequently, if \((\hat{\mbfU}^{K,1},\hat{V}^K)\) is chosen uniformly from \(\{(\bar{\mbfU}^{1,1}, V^1), \ldots, (\bar{\mbfU}^{K,1}, V^K)\}\), then we have
\[
\E\dist^2\left(0,\partial \Phi(\hat{\mbfU}^{K,1}, \hat{V}^K)\right)  \leq \frac{C(\mathcal L^0 - \alpha)}{K} = \mathcal{O}(1/K).
\]
In other words, the number of communication rounds \(K\) needed to obtain an \(\epsilon\)-stationary point of \(\Phi\) is at most 
\[
K = \frac{C(\mathcal L^0 - \alpha)}{\epsilon^2} = \mathcal{O}(1/\epsilon^2).
\]
\end{theorem}

\begin{proof}
From the update rule of Algorithm~\ref{FedLMFmm} for \(V^{k+1}\) in \eqref{s:update}, it follows that for all \(k \geq 0\),
\begin{equation}\label{1127a}
\nu^{k+1} := \sum_{i=1}^p\left[Y_i^{k+1} + \beta(W_i^{k+1} - V^{k+1})\right] \in \partial_V [R+ \mathcal I_{\mathcal V}](V^{k+1}).
\end{equation}
Thus, by invoking Lemma~\ref{convexity-smoothness} (b), it follows that
\begin{equation}
    \frac{1}{p}\sum_{i=1}^p\nabla_V f_i(\bar U_i^{k+1,1}, V^{k+1}) + \nu^{k+1} \in \partial_V\Phi(\bar\mbfU^{k+1,1}, V^{k+1}).
\end{equation}
Additionally, from Lemma~\ref{lemmay}, we have
\begin{equation}\label{1217a}
    Y_i^{k+1} = - \frac{1}{p}\left[\nabla_V f_i(U_i^{k+1}, W_i^{k,N-1}) + L_{U_i^{k+1}}(W_i^{k+1} - W_i^{k,N-1})\right].
\end{equation}
From the definition of \(\bar U_i^{k+1,1}\) (see \eqref{1122g}), we have
\begin{equation}\label{1127b}
\zeta_i^{k+1} := -\nabla_Uf_i(U_i^{k+1},W_i^{k+1}) - L_{W_i^{k+1}}(\bar U_i^{k+1,1} - U_i^{k+1}) \in \partial_{U_i} [R_i+ \mathcal I_{\mathcal U_i}](\bar U_i^{k+1,1}).
\end{equation}
Thus, by invoking Lemma~\ref{convexity-smoothness} (b), it follows that
\begin{equation*}
    \frac{1}{p}[\nabla_{U_i}f_i(\bar U_i^{k+1,1}, V^{k+1}) + \zeta_i^{k+1}] \in \partial_{U_i}\Phi(\bar\mbfU^{k+1,1}, V^{k+1},).
\end{equation*}
Therefore, we arrive at
\begin{equation}\label{122d}
    \begin{split}
       \dist\left(0,\partial \Phi(\bar\mbfU^{k+1,1}, V^{k+1})\right)^2 & \leq \left\|\frac{1}{p}\sum_{i=1}^p\nabla_V f_i(\bar U_i^{k+1,1}, V^{k+1}) + \nu^{k+1}\right\|^2 \\
       & \quad + \sum_{i=1}^p\left\|\frac{1}{p}[\nabla_{U_i}f_i(\bar U_i^{k+1,1}, V^{k+1}) + \zeta_i^{k+1}]\right\|^2.
    \end{split}
\end{equation}
On the other hand, from \eqref{1127a} and \eqref{1217a} we have
\begin{equation}
    \begin{split}
        &\left\|\frac{1}{p}\sum_{i=1}^p\nabla_V f_i(\bar U_i^{k+1,1}, V^{k+1}) + \nu^{k+1}\right\| \\
        = & \Bigg\|\frac{1}{p}\sum_{i=1}^p\big[\nabla_V f_i(\bar U_i^{k+1,1}, V^{k+1}) - \nabla_V f_i(U_i^{k+1}, W_i^{k,N-1})  - L_{U_i^{k+1}}(W_i^{k+1} - W_i^{k,N-1}) \\
        & \quad + p\beta(W_i^{k+1} - V^{k+1})\big]\Bigg\| \\
        \leq & \sum_{i=1}^p\bigg[L/p\|\bar U_i^{k+1,1} - U_i^{k+1}\| + L_{U_i^{k+1}}/p\|V^{k+1}-W_i^{k,N-1}\| + L_{U_i^{k+1}}/p\|W_i^{k+1}-W_i^{k,N-1}\| \\
        & \quad + \beta\|W_i^{k+1} - V^{k+1}\|\bigg] \\
        \leq & \sum_{i=1}^p\bigg[L/p\|\bar U_i^{k+1,1} - U_i^{k+1}\| + (\bar L_U/p+\beta)\|V^{k+1}-W_i^{k+1}\| + 2\bar L_U/p\|W_i^{k+1}-W_i^{k,N-1}\|\bigg].
    \end{split}
\end{equation}
Thus, we obtain
\begin{equation}\label{1128b}
    \begin{split}
        & \left\|\frac{1}{p}\sum_{i=1}^p\nabla_V f_i(\bar U_i^{k+1,1}, V^{k+1}) + \nu^{k+1}\right\|^2 \\
        \leq & 3\sum_{i=1}^p\bigg[ L^2/p\|\bar U_i^{k+1,1} - U_i^{k+1}\|^2 + (\bar L_U+p\beta)^2/p\|V^{k+1}-W_i^{k+1}\|^2 + 4\bar L_U^2/p\|W_i^{k+1}-W_i^{k,N-1}\|^2\bigg],
    \end{split}
\end{equation}
It follows from \eqref{1128a} that
\begin{align}
    \sum_{k=0}^K\sum_{i=1}^p\sum_{l=1}^N\E\|W_i^{k,l}-W_i^{k,l-1}\|^2 & \leq \frac{2p^2\beta(\mathcal L^0 - \alpha)}{p_{\min}p\beta\underaccent{\bar} L_U - 8(8+N)\bar L_U^2} \label{1217b}\\
    \sum_{k=0}^K\sum_{i=1}^p\sum_{l=1}^N\E\|U_i^{k,l}-U_i^{k,l-1}\|^2 & \leq \frac{2p^2\beta(\mathcal L^0 - \alpha)}{p\beta\underaccent{\bar} L_V - 8NL^2}\label{1217c}.
\end{align}
From \eqref{1126c}, \eqref{1217b}, and \eqref{1217c}, we arrive at
\begin{equation}\label{1217d}
  \sum_{k=0}^K\sum_{i=1}^p  \E\|W_i^k-V^k\|^2 \leq \biggl[\frac{4(N\beta^2p^2 + 4(8+N)\bar L_U^2)}{p_{\min}p\beta\underaccent{\bar} L_U - 8(8+N)\bar L_U^2} + \frac{16NL^2}{p\beta\underaccent{\bar} L_V - 8NL^2} \biggr]\frac{\mathcal L^0 - \alpha}{p_{\min}\beta}
\end{equation}
It follows from \eqref{1216t} and \eqref{1217c} that
\begin{equation}\label{1217e}
    \sum_{k=0}^K\sum_{i=1}^p\E\|\bar U_i^{k,1}-U_i^k\|^2 \leq \frac{2p^2\beta(\mathcal L^0 - \alpha)}{p_{\min}(p\beta\underaccent{\bar} L_V - 8NL^2)}
\end{equation}
By taking the total expectation of \eqref{1128b}, telescoping over \(k = 0, \ldots, K-1\), and using \eqref{1217b}, \eqref{1217d}, and \eqref{1217e}, we arrive at
\begin{equation}\label{122b}
    \begin{split}
        \sum_{k=0}^{K-1}\E\left\|\frac{1}{p}\sum_{i=1}^p\nabla_V f_i(\bar U_i^{k+1,1}, V^{k+1}) + \nu^{k+1}\right\|^2
        \leq &\ \biggl[\frac{12(N\beta^2p^2 + 4(8+N)\bar L_U^2)(\bar L_U^2 +p\beta)^2 + 24p_{\min}p^2\beta^2\bar L_U^2}{p_{\min}p\beta\underaccent{\bar} L_U - 8(8+N)\bar L_U^2}\\
        & + \frac{6(8N(\bar L_U^2 +p\beta)^2 + p^2\beta^2)L^2}{p\beta\underaccent{\bar} L_V - 8NL^2} \biggr]\frac{\mathcal L^0 - \alpha}{p_{\min}p\beta}.
    \end{split}
\end{equation}

Similarly, from \eqref{1127b}, we have
\[
\begin{split}
    \left\|\frac{1}{p}[\nabla_{U_i}f_i(\bar U_i^{k+1,1}, V^{k+1}) + \zeta_i^{k+1}]\right\| & = \ \frac{1}{p}\left\|\nabla_{U_i}f_i(\bar U_i^{k+1,1}, V^{k+1}) - \nabla_{U_i}f_i(U_i^{k+1}, W_i^{k+1}) - L_{W_i^{k+1}}(\bar U_i^{k+1,1}- U_i^{k+1})\right\|\\
    & \leq  2\bar L_V/p\|\bar U_i^{k+1,1}- U_i^{k+1}\| + \kappa/p\|V^{k+1}-W_i^{k+1}\|.
\end{split}
\]
This implies that
\[
\begin{split}
     \left\|\frac{1}{p}[\nabla_{U_i}f_i(\bar U_i^{k+1,1}, V^{k+1}) + \zeta_i^{k+1}]\right\|^2 &  \leq  8\bar L_V^2/p^2\|\bar U_i^{k+1,1}- U_i^{k+1}\|^2 + 2\kappa^2/p^2\|V^{k+1}-W_i^{k+1}\|^2
    \end{split}
\]
This together with \eqref{1217d} and \eqref{1217e} gives us that
\begin{equation}\label{122c}
    \begin{split}
         \sum_{k=0}^{K-1}\sum_{i=1}^p\E\|\frac{1}{p}[\nabla_{U_i}f_i(\bar U_i^{k+1,1}, V^{k+1}) + \zeta_i^{k+1}]\|^2 
        \leq &\ \biggl[\frac{8(N\beta^2p^2 + 4(8+N)\bar L_U^2)\kappa^2}{p_{\min}p\beta\underaccent{\bar} L_U - 8(8+N)\bar L_U^2}\\
        & + \frac{32NL^2\kappa^2 + 16p^2\beta^2\bar L_V^2}{p\beta\underaccent{\bar} L_V - 8NL^2} \biggr]\frac{\mathcal L^0 - \alpha}{p_{\min}p^2\beta}.
    \end{split}
\end{equation}
By combining \eqref{122b}, \eqref{122c}, and \eqref{122d}, we conclude
\begin{equation*}
    \begin{split}
        \sum_{k=0}^{K-1}\E \dist\left(0,\partial \Phi(\bar\mbfU^{k+1,1}, V^{k+1})\right)^2 
        \leq&\ \biggl[\frac{4(N\beta^2p^2 + 4(8+N)\bar L_U^2)(3(\bar L_U^2 +p\beta)^2p + 2\kappa^2) + 24p_{\min}p^3\beta^2\bar L_U^2}{p_{\min}p\beta\underaccent{\bar} L_U - 8(8+N)\bar L_U^2}\\
        & + \frac{6(8N(\bar L_U^2 +p\beta)^2 + p^2\beta^2)L^2p + 32NL^2\kappa^2 + 16p^2\beta^2\bar L_V^2}{p\beta\underaccent{\bar} L_V - 8NL^2} \biggr]\frac{\mathcal L^0 - \alpha}{p_{\min}p^2\beta}.
    \end{split}
\end{equation*}
This completes the proof.
\end{proof}

\section{Numerical experiments}\label{sec:exp}

We evaluate the convergence behavior and performance of our proposed algorithm, Algorithm \ref{FedLMFmm} (\texttt{FedMC-ADMM}), for the federated MC problem. To assess its effectiveness, we compare \texttt{FedMC-ADMM} with \texttt{FedMAvg} \cite{wang2022federated}, using objective values and test Root Mean Square Errors (RMSE) as our performance metrics. It is important to note that we do not compare with \texttt{FedMGS} \cite{wang2022federated} since its gradient-sharing framework requires clients to send gradients to the server. In federated matrix completion, this approach does not protect the client's data privacy. Specifically, the server could infer the rating data of the clients directly from the gradients, as demonstrated in \cite{chai2020secure}.

All the algorithms are initialized with random values for $(\mbfU^0, V^0)$, where the entries are sampled from a uniform distribution over $[0, 1]$. The experiments are conducted on a Windows workstation with configurations: 13th Gen Intel(R) Core(TM) i9-13900K 3.00 GHz processors and 128GB RAM. The implementation is written in Python and C, and the source code is available at \url{https://github.com/nhatpd/FedMC-ADMM}.


\subsection{Data sets and the general method}
We conducted our experiments on two widely used datasets in recommendation systems, MovieLens and Netflix, both of which contain user rating data. 
The characteristics of these datasets are summarized in Table~\ref{dataset}. We denote an $m\times n$ matrix $M$ for these datasets, where incompleted entries are replaced by zeros. For the experiments, each dataset was distributed among 100 clients to simulate a FL environment. Specifically, the user-item rating matrix from each dataset was divided such that each client was assigned a non-overlapping subset of users along with their corresponding ratings. Each client owns a private and non-overlapping $m_i\times n $ matrix $M_i$ containing rows of $M$ with $m_1+m_2+\ldots+m_{100}=m$. The rank parameter was set to $r = 5$, $8$, and $13$ for the MovieLens 1M, MovieLens 10M, and Netflix datasets, respectively. 

For training and testing, the observed ratings were randomly partitioned, with 80\% used for training and 20\% reserved for testing. Each algorithm was executed for 100 communication rounds for all datasets. First, we run our algorithm \texttt{FedMC-ADMM} to solve the optimization problem \eqref{model} with
\[
F_i(M_i,U_iV_i)=\frac{1}{2}\|\mathcal{P}_{\Omega_i}(M_i-U_iW_i)\|^2+R_i(U_i),
\]
where $\Omega_i=\{(t,j)\in [m_i]\times [n]|\, (M_i)_{tj}\neq 0 \mbox{ and }(M_i)_{tj}\; \mbox{is in the training set}\}$ and $\mathcal{P}_{\Omega_i}:\mathbb{R}^{m_i\times n}\to \mathbb{R}^{m_i\times n}$ is the linear operator defined by $(P_{\Omega_i}Z)_{tj}=Z_{tj}$ if $(t,j)\in \Omega_i$ and $0$ otherwise for $Z\in \mathbb{R}^{m_i\times n}$. 

The evaluation at each iteration focuses on two metrics: the training objective value and RMSE computed on the test set, defined as
\[
RMSE = \sqrt{\frac{\sum_{i=1}^p \|\mathcal{P}_{T_i}(M_i - U_iV)\|^2}{N_T}},
\]
where  $T_i=\{(t,j)\in [m_i]\times [n]|\, (M_i)_{tj}\neq 0 \mbox{ and }(M_i)_{tj}\; \mbox{is in the testing set}\}$ and $N_T$ is the number of observed ratings in the testing set.

\begin{table}[!htbp]
\centering
\caption{The number of users, items, and ratings included in each dataset.}\label{dataset}
\begin{tabular}{@{}cllll@{}}
\toprule
\multicolumn{2}{l}{Dataset}       & \multicolumn{1}{l}{\#users} & \multicolumn{1}{l}{\#items} & \multicolumn{1}{l}{\#ratings} \\ \midrule
\multirow{2}{*}{MovieLens}  & 1M  & 6,040                       & 3,449                       & 999,714                       \\
                            & 10M & 69,878                      & 10,677                      & 10,000,054                    \\
\multicolumn{1}{l}{Netflix} &     & 480,189                     & 17,770                      & 100,480,507                   \\ \hline
\end{tabular}
\end{table}

\subsection{Comparison between \texttt{FedMC-ADMM} and \texttt{FedMAvg}}

In this experiment, we utilize $\ell_2$-norm squared regularization terms as follows
\[
R_i(U_i) = \frac{\lambda}{2}\|U_i\|_F^2 \quad \text{and} \quad R(V) = \frac{\gamma}{2}\|V\|_F^2,
\]
where the regularization parameters are fixed at $\lambda = \gamma = 10^{-6}$. 
The update rule for \(U_i^{k,l}\) in \eqref{u:update} is given by
\begin{equation}\label{1217f}
    U_i^{k,l} \in \argmin_{U_i \in \mathbb R^{m_i\times r}}\; \langle \mathcal{P}_{\Omega_i}(U_i^{k,l-1} W_i^k - M_i)(W_i^k)^T, U_i \rangle + \frac{L_{W_i^k}}{2} \|U_i - U_i^{k,l-1}\|^2 + \frac{\lambda}{2} \|U_i\|_F^2,
\end{equation}
where \(L_{W_i^k} = \|W_i^k (W_i^k)^T\|_F\) as discussed in Remark~\ref{RemP}. The closed-form solution of \eqref{1217f} is
\begin{equation}
    U_i^{k,l} = \frac{U_i^{k,l-1} L_{W_i^k} - \mathcal{P}_{\Omega_i}(U_i^{k,l-1} W_i^k - M_i)(W_i^k)^T}{L_{W_i^k} + \lambda}.
\end{equation}
The update rule for \(W_i^{k,l}\) in \eqref{v:update} is
\begin{equation}\label{1217h}
    W_i^{k,l} = \frac{L_{U_i^{k+1}}/p W_i^{k,l-1} + \beta V^k - (U_i^{k+1})^T \mathcal{P}_{\Omega_i}(U_i^{k+1} W_i^{k,l-1} - M_i)/p - Y_i^k}{L_{U_i^{k+1}}/p + \beta}
\end{equation}
with  \( L_{U_i^{k+1}} = \|(U_i^{k+1})^T U_i^{k+1}\|_F \). Finally, \(V^{k+1}\) in \eqref{s:update} is computed by
\begin{equation}
    V^{k+1} = \frac{\sum_{i=1}^p [\beta W_i^{k+1} + Y_i^{k+1}]}{p \beta + \gamma}.
\end{equation}

The algorithm \texttt{FedMAvg} with hyper-parameters proposed in \cite{wang2022federated} is recalled below. 
At each communication round, \texttt{FedMAvg} updates \(U_i^{k+1}\) for all clients \(i \in [p]\) via the following iterative process
\[U_i^{k,l} = U_i^{k,l-1} - \frac{ \mathcal{P}_{\Omega_i}\left(U_i^{k,l-1} V^k - M_i\right)(V^k)^\top + \lambda U_i^{k,l-1} }{c^k}
\]
for \(l = 1, \dots, Q_1\) and \(U_i^{k,0} = U_i^k\). After \(Q_1\) iterations, \(U_i^{k+1}\) is set as \(U_i^{k,Q_1}\). Next, \texttt{FedMAvg} updates \(W_i^{k+1}\) through
\[
W_i^{k,l} = W_i^{k,l-1} - \frac{(U_i^{k+1})^\top \mathcal{P}_{\Omega_i}\left(U_i^{k+1} W_i^{k,l-1} - M_i\right)/p + \gamma W_i^{k,l-1} }{d_i^k},
\]
for \(l = 1, \dots, Q_2\) and $W_i^{k,0} = V^k$. Upon completing \(Q_2\) steps, \(W_i^{k+1}\) is set as \(W_i^{k,Q_2}\). The step size \(d_i^k\) is chosen as
\[
d_i^k = 5 \lambda_{\max}\left((U_i^{k+1})^\top U_i^{k+1}\right).
\]
At the server, \texttt{FedMAvg} aggregates client updates to compute \(V^{k+1}\)
\[
V^{k+1} = \frac{1}{m}\sum_{i \in S_k} W_i^{k+1},
\]
where \(S_k\) denotes the subset of selected clients at round \(k\), and \(m = |S_k|\). The number of inner iterations is set to \(Q_1 = Q_2 = 10\). For the algorithm \texttt{FedMC-ADMM}, the inner iteration number is fixed similarly at \(N = 10\).

During each communication round, 10 users are randomly sampled to participate in the updates. The average values of the objective function and RMSE over communication rounds are shown in Figure~\ref{exp1}.

\begin{figure*}[!htpb] 
\vspace{-1ex}
\begin{center}
\begin{tabular}{cc}
\includegraphics[width=0.45\linewidth]{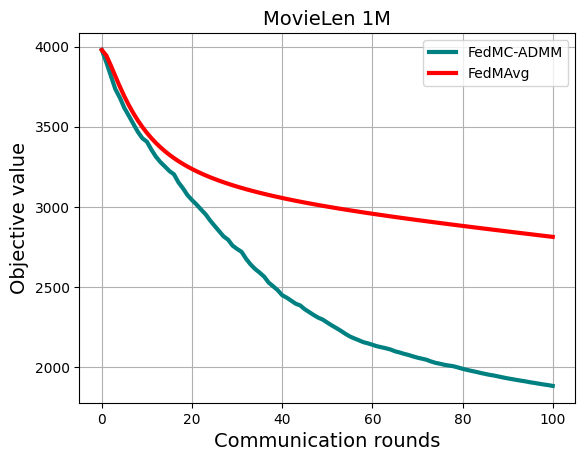} & 
\includegraphics[width=0.45\linewidth]{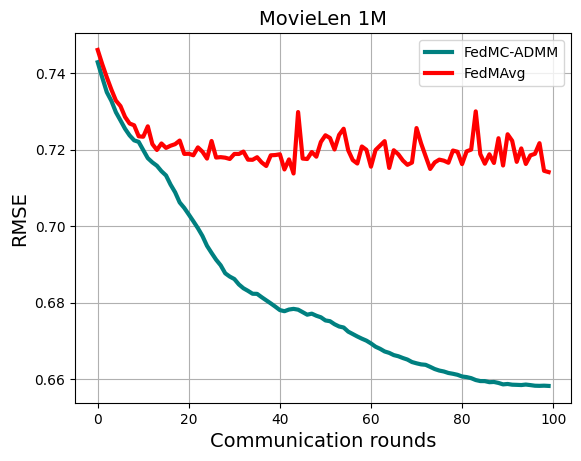}\\
\includegraphics[width=0.45\linewidth]{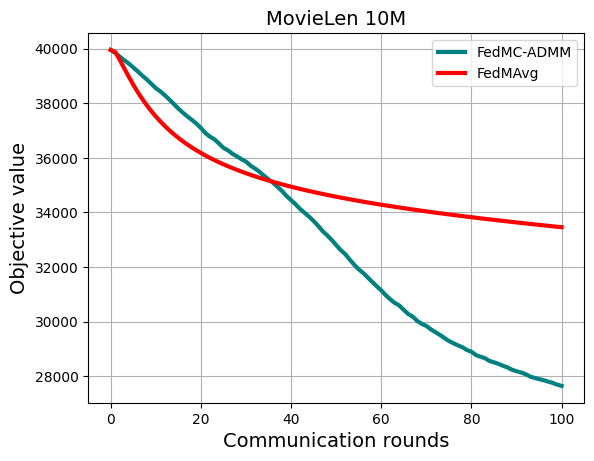} & 
\includegraphics[width=0.45\linewidth]{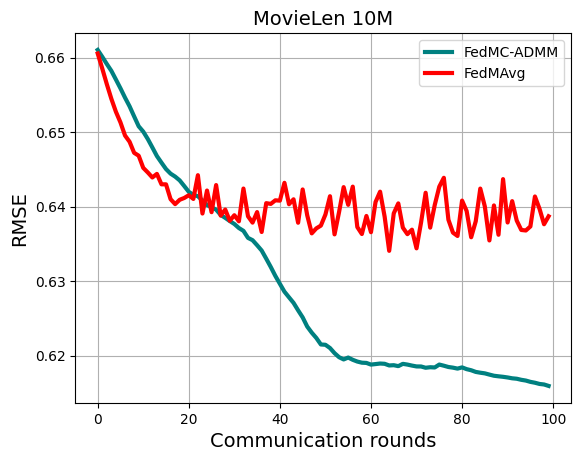}\\
\includegraphics[width=0.45\linewidth]{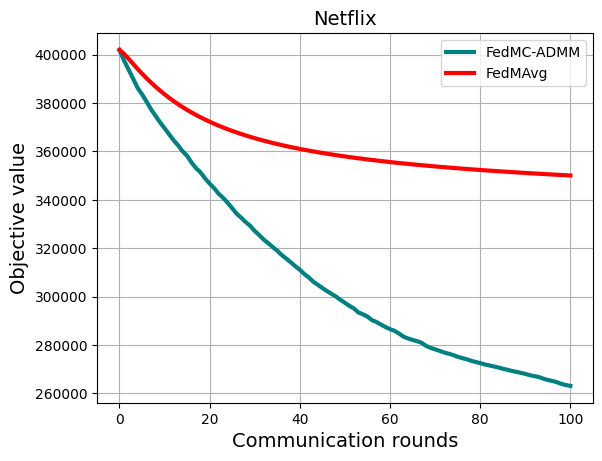} &
\includegraphics[width=0.45\linewidth]{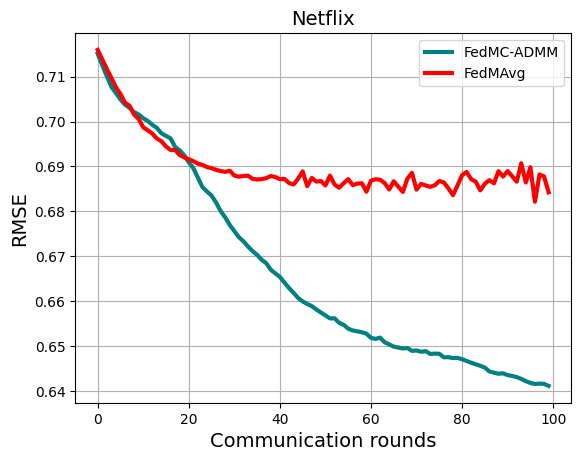}
\end{tabular}
\caption{Evolution of the training objective value and testing RMSE with respect to the communication round of \texttt{FedMC-ADMM} and \texttt{FedMAvg} on Movielens 1M, 10M, and Netflix. } \label{exp1}
\end{center}
\vspace{-1ex}
\end{figure*}
As shown in Figure \ref{exp1}, \texttt{FedMC-ADMM} consistently outperforms \texttt{FedMAvg} across all datasets. In particular, the left column demonstrates that \texttt{FedMC-ADMM} achieves faster and more significant reductions in the objective value compared to \texttt{FedMAvg}, particularly during the last communication rounds. Meanwhile, the right column illustrates that \texttt{FedMC-ADMM} achieves lower RMSE on the test set across all datasets, which highlights its superior generalization capabilities. Importantly, we observed that as the dataset size increases, from MovieLens 1M to Netflix, the performance gap between \texttt{FedMC-ADMM} and \texttt{FedMAvg} becomes more pronounced, indicating that \texttt{FedMC-ADMM} scales more effectively with larger datasets.

\subsection{Effect of number of inner iterations}

In this experiment, we investigate the effect of varying the number of inner iterations on the performance of \texttt{FedMC-ADMM}. For this purpose, we consider \(\ell_1\)-norm regularization terms, defined as follows
\begin{equation}\label{l1-norm}
    R_i(U_i) = \lambda \|U_i\|_1 \quad \text{and} \quad R(V) = \gamma \|V\|_1,
\end{equation}
where the regularization parameters are fixed at \(\lambda = \gamma = 10^{-6}\). We also employ the local loss \(f_i(U_i, W_i) = \frac{1}{2} \|\mathcal{P}_{\Omega_i}(M_i - U_i W_i)\|^2\).

The update rule for \(U_i^{k,l}\) in \eqref{u:update} is given by
\begin{equation}\label{1217g}
    U_i^{k,l} \in \argmin_{U_i \in \mathbb R^{m_i\times r}} \langle \mathcal{P}_{\Omega_i}(U_i^{k,l-1} W_i^k - M_i)(W_i^k)^T, U_i \rangle + \frac{L_{W_i^k}}{2} \|U_i - U_i^{k,l-1}\|^2 + \lambda \|U_i\|_1,
\end{equation}
where \(L_{W_i^k} = \|W_i^k (W_i^k)^T\|_F\) as discussed in Remark~\ref{RemP}. The closed-form solution of \eqref{1217g} is obtained using the soft-thresholding operator \(\mathcal{S}\), and is given by
\begin{equation}
    U_i^{k,l} = \mathcal{S}\left(U_i^{k,l-1} -\mathcal{P}_{\Omega_i}(U_i^{k,l-1} W_i^k - M_i)(W_i^k)^T/L_{W_i^k}, \lambda/L_{W_i^k}\right),
\end{equation}
where the soft-thresholding operator \(\mathcal{S}\) is defined as
\begin{equation}
    \mathcal{S}(Q, \tau)_{tj} = \left[ |Q_{tj}| - \tau \right]_+ \operatorname{sign}(Q_{tj}).
\end{equation}

The update rule for \(W_i^{k,l}\) in \eqref{v:update} remains the same as in \eqref{1217h}. Finally, the update rule for \(V^{k+1}\) in \eqref{s:update} is
\begin{equation}
    V^{k+1} = \mathcal{S}\left(\sum_{i=1}^p \left[ W_i^{k+1} + Y_i^{k+1}/\beta \right]/p, \gamma/(p\beta)\right).
\end{equation}

To examine the impact of the number of inner iterations $N$ on the algorithm's convergence, we varied $N$ from 5 to 30. Figure~\ref{exp2} illustrates how the average objective function value and RMSE evolve with training time, which highlights the relationship between $N$ and convergence behavior.

Figure~\ref{exp2} shows that smaller values of $N$ result in faster convergence of the objective value. This effect is particularly pronounced for $N = 5$ and $N = 10$, where the objective value decreases rapidly compared to larger $N$ values. These results suggest that fewer inner iterations can accelerate the algorithm's convergence in terms of training time. The right plot further demonstrates that smaller $N$ values achieve lower RMSE on the test set more quickly, indicating improved generalization within a shorter training period. However, as $N$ increases, the improvements in RMSE diminish, and the computational cost increases. This underscores the efficiency of choosing smaller $N$ values in scenarios where fast convergence is a priority.

\begin{figure*}[!htpb] 
\vspace{-1ex}
\begin{center}
\begin{tabular}{cc}
\includegraphics[width=0.45\linewidth]{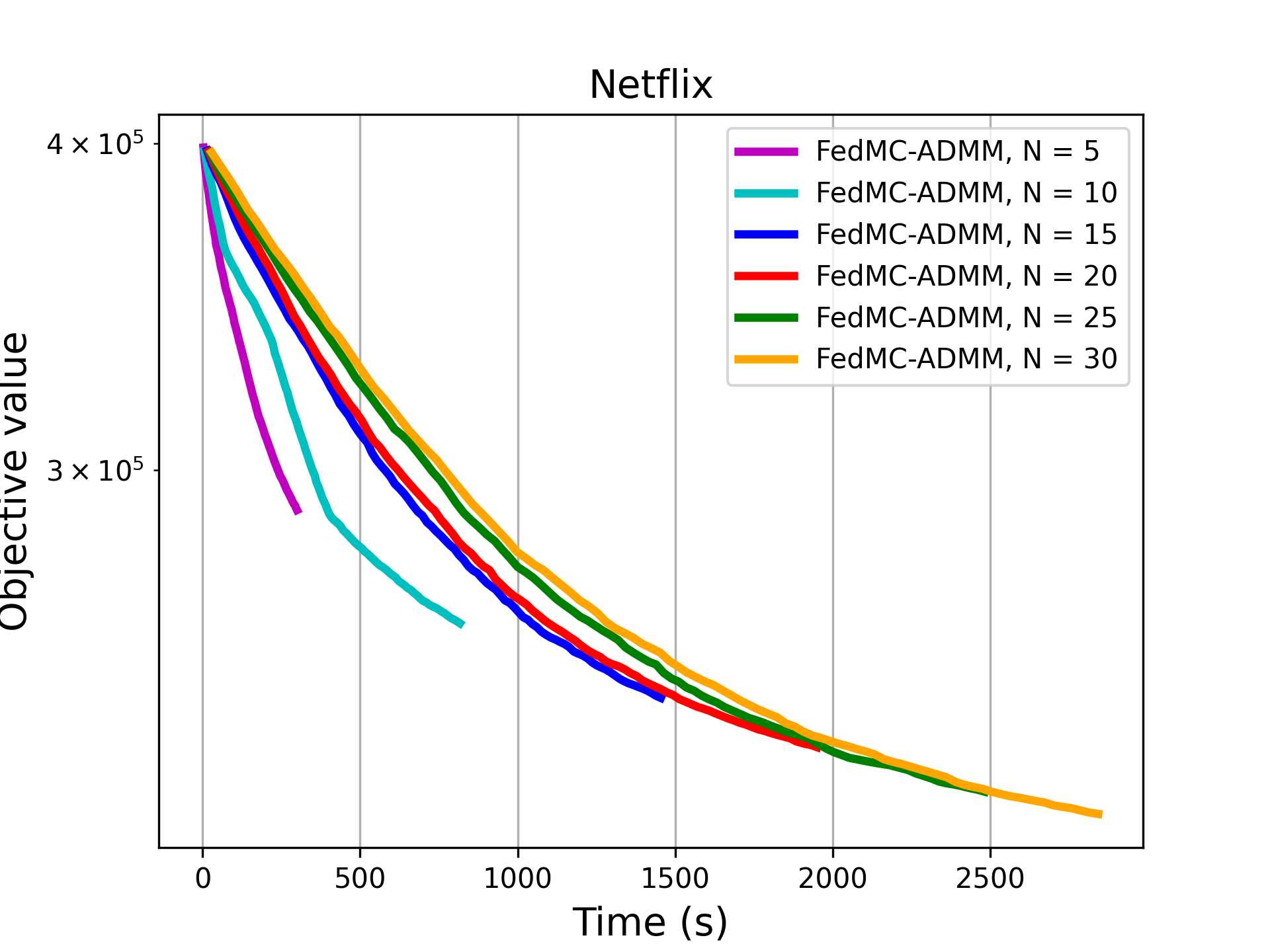} & 
\includegraphics[width=0.45\linewidth]{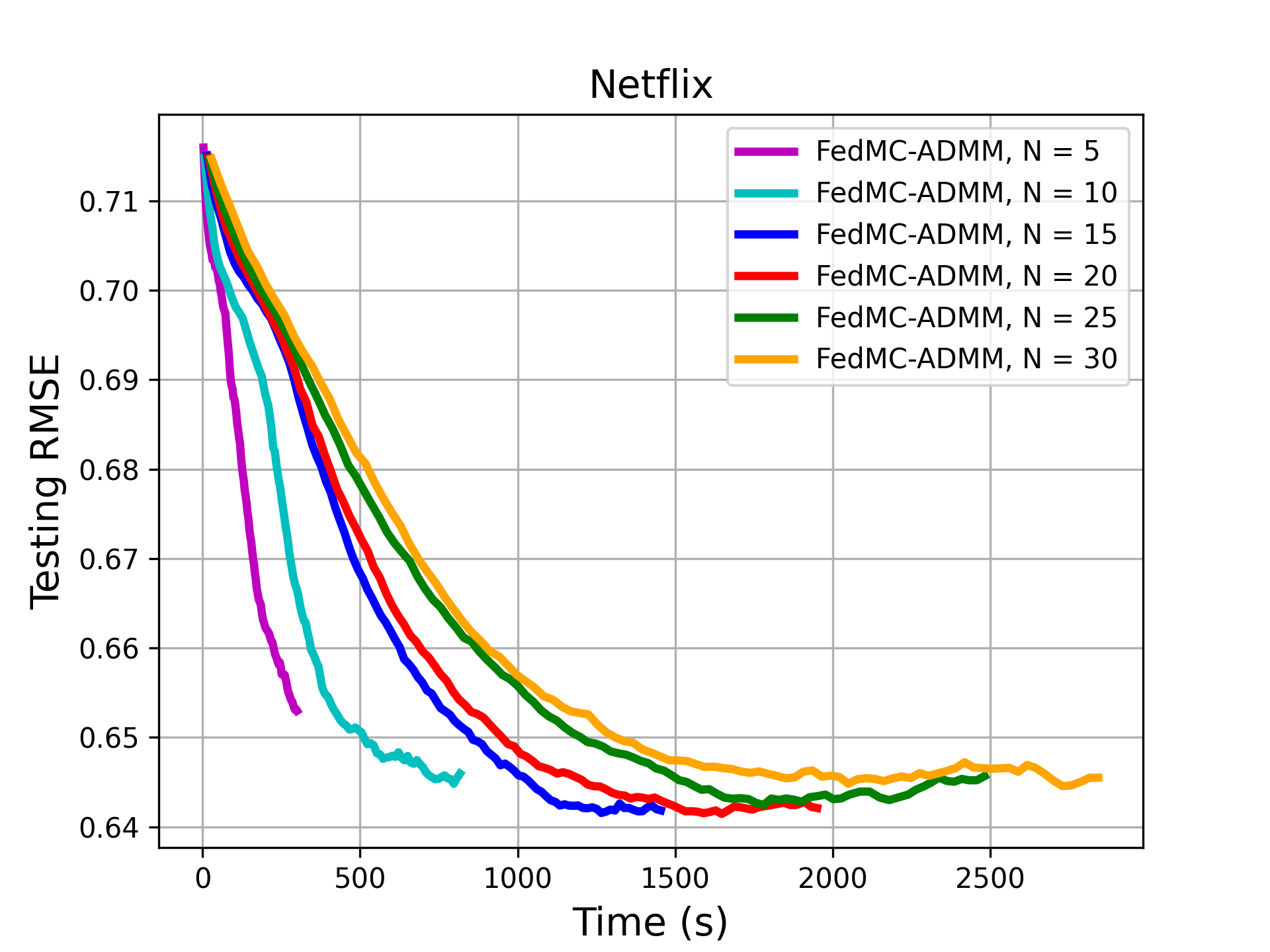}
\end{tabular}
\caption{Effect of the number of inner iterations ($N$) on the convergence of \texttt{FedMC-ADMM} for the Netflix dataset. The left panel shows the evolution of the objective value, while the right panel presents the testing RMSE as a function of training time. Results are shown for $N = 5, 10, 15, 20, 25, 30$. } \label{exp2}
\end{center}
\vspace{-1ex}
\end{figure*}

\subsection{Effect of regularization parameter}

In this experiment, we revisit the $\ell_1$-norm regularization terms, defined in \eqref{l1-norm}, to study the impact of the regularization parameter on the performance of \texttt{FedMC-ADMM}. The regularization parameters are set as $\gamma = 10 \cdot \lambda \cdot \beta$, with $\lambda$ varying across the values $\{1\text{e-}6, 1\text{e-}5, 1\text{e-}4, 1\text{e-}3, 1\text{e-}2, 1\text{e-}1\}$. 

Figure~\ref{exp3} presents three plots illustrating the effect of varying $\lambda$ on the convergence of the objective value (left), the testing RMSE (center), and the average proportion of nonzero entries (\texttt{nnz}) in the matrices \(\mathbf{U}\) and \(V\) (right).

From Figure~\ref{exp3}, we observe that smaller regularization parameters ($\lambda \leq 0.001$) provide comparable results in terms of the training objective value and testing RMSE. We also learned that among these, $\lambda = 0.001$ strikes the best balance by minimizing the objective value, maintaining a low testing RMSE, and preserving matrix sparsity. In contrast, larger $\lambda$ values (e.g., $\lambda = 0.1$) lead to solutions where the matrices $\mathbf{U}$ and $V$ quickly become entirely sparse after only a few communication rounds.

\begin{figure*}[!htpb] 
\vspace{-1ex}
\begin{center}
\begin{tabular}{ccc}
\includegraphics[width=0.33\linewidth]{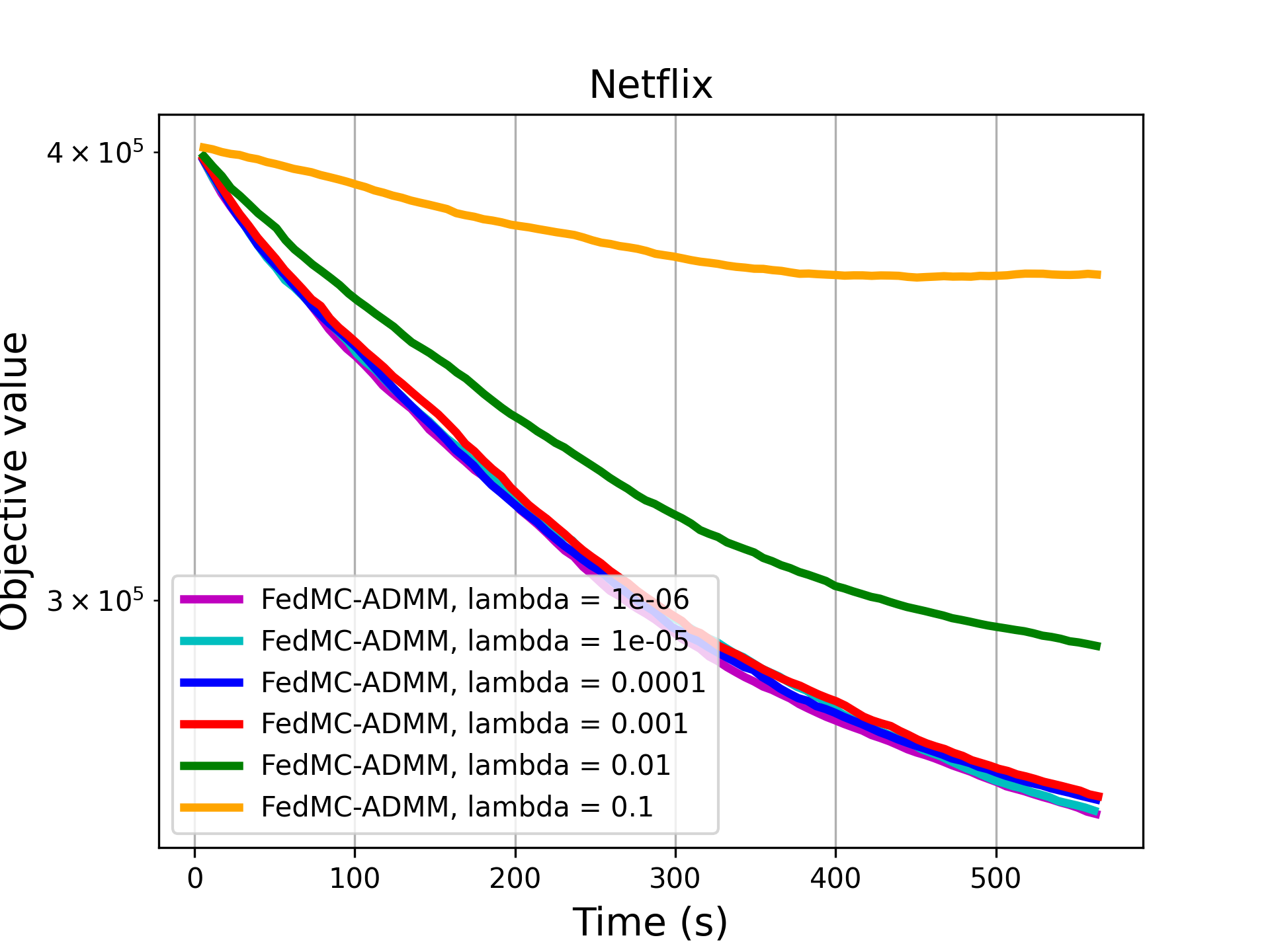} & 
\includegraphics[width=0.33\linewidth]{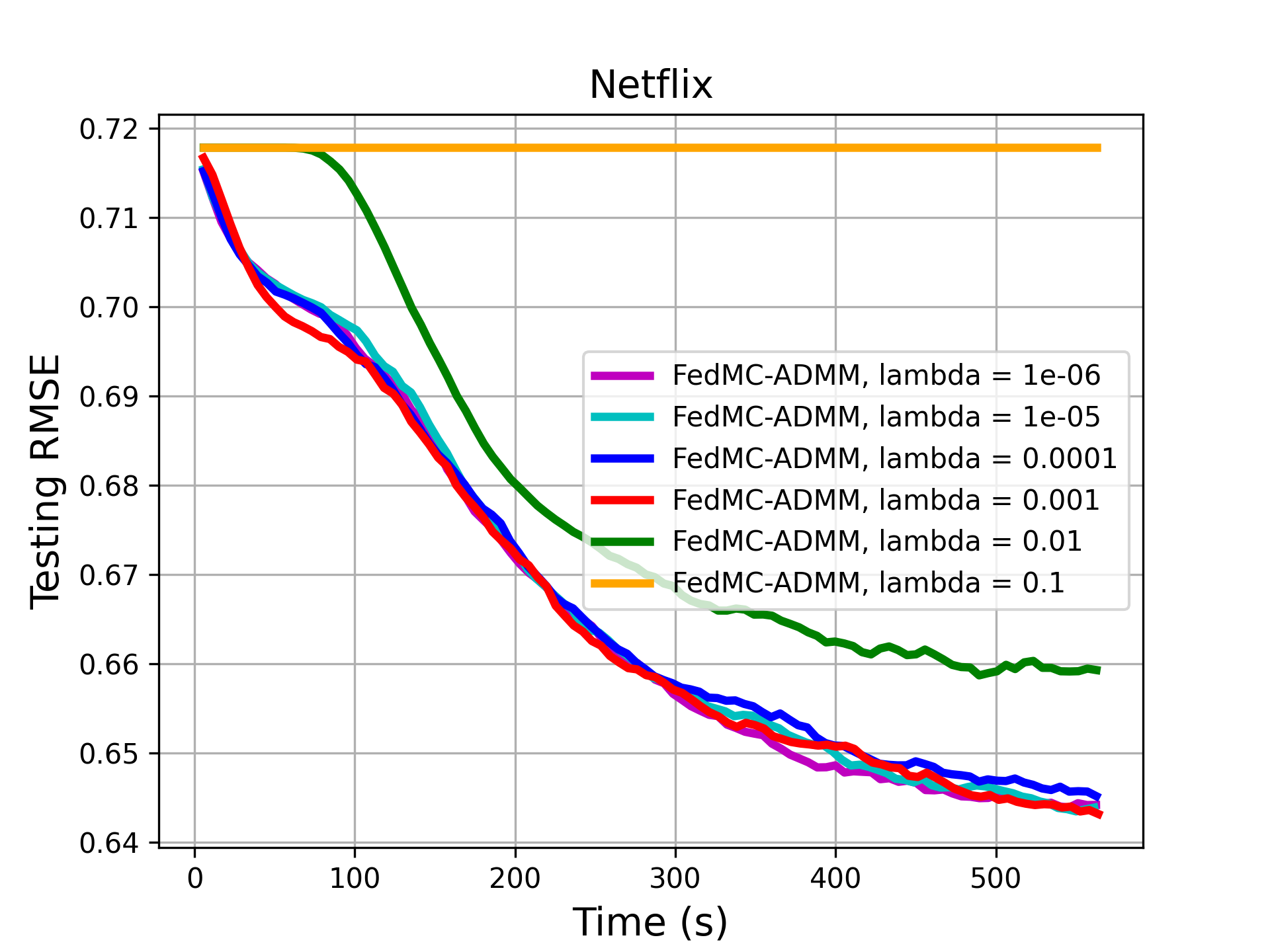} & 
\includegraphics[width=0.3\linewidth]{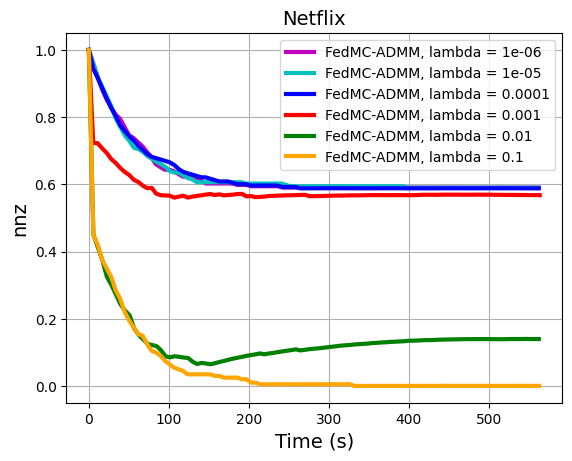}
\end{tabular}
\caption{Effect of the regularization parameter $\lambda$ on the performance of \texttt{FedMC-ADMM} for the Netflix dataset. The left plot shows the evolution of the objective value, the center plot presents the testing RMSE as a function of training time, and the right plot illustrates the average number of nonzero entries (\texttt{nnz}) in the matrices \(\mathbf{U}\) and \(V\) over time. Results are shown for $\lambda = 1\text{e-}6, 1\text{e-}5, 1\text{e-}4, 1\text{e-}3, 1\text{e-}2, 1\text{e-}1$.} \label{exp3}
\end{center}
\vspace{-1ex}
\end{figure*}

\section{Conclusion}\label{sec:conclusion}
In this work, we introduced \texttt{FedMC-ADMM}, a novel algorithmic framework that integrates the Alternating Direction Method of Multipliers with a randomized block-coordinate strategy and alternating proximal gradient steps to address the federated matrix completion problem. Our approach is designed to handle the challenges of nonconvex, nonsmooth, and multi-block optimization problems that arise in federated matrix completion problems. By leveraging a randomized block-coordinate strategy, \texttt{FedMC-ADMM} effectively mitigates communication bottlenecks inherent in FL while preserving privacy guarantees for user data.

Theoretically, we established the convergence properties of \texttt{FedMC-ADMM} under mild assumptions. Specifically, we proved that the algorithm almost surely achieves subsequential convergence to a stationary point. We further demonstrated that \texttt{FedMC-ADMM} achieves a convergence rate of $\mathcal{O}(K^{-1/2})$ and a communication complexity of $\mathcal{O}(\epsilon^{-2})$ to reach an $\epsilon$-stationary point, matching the best-known theoretical bounds in the literature.

Empirically, we validated the effectiveness of  \texttt{FedMC-ADMM} by applying it to the federated matrix completion problem in federated recommendation systems. Our experimental results on real-world datasets, including MovieLens 1M, 10M, and Netflix, showed that \texttt{FedMC-ADMM} consistently outperforms existing FL algorithms in terms of convergence speed and testing RMSE.  Furthermore, we observed that the advantages of \texttt{FedMC-ADMM} become increasingly pronounced as dataset sizes grow, underscoring its scalability and suitability for large-scale federated applications.

This work not only highlights the advantages of integrating ADMM with randomized block-coordinate updates and proximal steps but also sets a new benchmark for federated optimization in matrix completion tasks. Future research could extend this framework to other domains within FL, such as graph-based learning or dynamic systems where challenges like nonconvexity, non-smoothness, and communication constraints also emerge.

\bibliographystyle{plainnat}



\end{document}